\documentclass[lettersize,journal]{IEEEtran}
\usepackage{amsmath,amsfonts}
\usepackage{algorithmic}
\usepackage{array}
\usepackage[dvipsnames]{xcolor}
\usepackage{textcomp}
\usepackage{stfloats}
\usepackage{url}

\usepackage{verbatim}
\usepackage{graphicx}
\def\BibTeX{{\rm B\kern-.05em{\sc i\kern-.025em b}\kern-.08em
    T\kern-.1667em\lower.7ex\hbox{E}\kern-.125emX}}
\usepackage{balance}

\usepackage{siunitx}  
\usepackage{amsthm}

\usepackage{pdfpages}
\usepackage{booktabs}

\usepackage{verbatim}
\usepackage{array}

\usepackage[hidelinks]{hyperref}

\makeatletter
\let\MYcaption\@makecaption
\makeatother

\usepackage[font=footnotesize]{subcaption}

\makeatletter
\let\@makecaption\MYcaption
\makeatother

\usepackage{flushend}

\newcommand{\vx}{\mathbf{x}}    
\newcommand{\vs}{\mathbf{s}}    
\newcommand{\vu}{\mathbf{u}}    

\newcommand{\seqX}{\mathbf{X}}    
\newcommand{\seqU}{\mathbf{U}}    
\usepackage{xcolor}

\newcommand{\norm}[1]{\left\lVert #1 \right\rVert}

\newcommand{\RR}{\mathbb{R}}   

\newcommand{\sX}{\mathcal{X}}   
\newcommand{\sU}{\mathcal{U}}   
\newcommand{\sM}{\mathcal{M}}   

\newcommand{\qsystem}[1]{\textit{#1}}

\newcommand{\vf}{\mathbf{f}}    
\newcommand{\vg}{\mathbf{g}}    

\DeclareMathOperator{\step}{step}

\newcommand{\theGx}{2.5}    
\newcommand{\theGy}{-1.2}    
\newcommand{\theGdeltax}{2.5}    
\newcommand{\theGdeltay}{-1.05}    

\newcommand{\tObs}{\fill[gray, opacity=0.5] (1, .5) rectangle (1.9, 1.6); \fill[gray, opacity=0.5] (1.2, -.8) rectangle (2.2, -.2);}

\newcommand{\tHead}{-{Latex[length=1.5mm]}}

\usepackage[ruled,vlined,linesnumbered]{algorithm2e}
\SetInd{0.25em}{0.5em}
\SetAlFnt{\footnotesize\sf}
\SetKwRepeat{Do}{do}{while}
\SetKw{KwStep}{step}
\SetKwInput{KwData}{Input}
\SetKwInput{KwResult}{Result}
\SetKwBlock{RepeatForever}{repeat}{}
\SetKw{Continue}{continue}
\SetKwComment{Comment}{$\triangleright$\ }{}
\SetCommentSty{itshape}

\usepackage[capitalise]{cleveref} 
\crefname{equation}{}{} 

\usepackage{tikz}
\usetikzlibrary{arrows.meta}
\usetikzlibrary{fit,calc}

\newtheorem{theorem}{Theorem}
\newtheorem{definition}{Definition}

\newtheorem{remark}{Remark}
\newtheorem{assumption}{Assumption}
\newtheorem{example}{Example}

\title{iDb-A*: Iterative Search and Optimization for Optimal Kinodynamic Motion Planning}
\author{Joaquim Ortiz-Haro, Wolfgang H\"onig, Valentin N. Hartmann, and Marc Toussaint
  \thanks{ \urlstyle{tt}
  Supplementary video: {\footnotesize \url{https://youtu.be/GoznOEWbUb8}}. Website: {\footnotesize\url{https://quimortiz.github.io/idbastar/}}. Code is available at {\footnotesize \url{https://github.com/quimortiz/dynoplan}}.}
  \thanks{All authors are with Faculty of Electrical Engineering and Computer Science, Technical University of Berlin, Berlin, Germany.}%
\thanks{This research has been supported by the German Research
Foundation (DFG) under Germany's Excellence Strategy –
EXC 2002/1–390523135 \emph{Science of Intelligence}, grant 448549715, and -- EXC 2120/1–390831618, as well as the
German-Israeli Foundation for Scientific Research (GIF), grant I-1491-407.6/2019.}%
}

\newcommand{\ALGdb}{\texttt{Db-A$^{*}$}\,}
\newcommand{\ALGidbas}{\texttt{iDb-A$^{*}$}\,}
\newcommand{\ALGsst}{\texttt{SST$^{*}$}\,}
\newcommand{\ALGrrt}{\texttt{RRT$^{*}$-TO}\,}

\begin{document}

\maketitle
\thispagestyle{empty}
\pagestyle{empty}

\begin{abstract}
	Motion planning for robotic systems with complex dynamics is a challenging problem.
	While recent sampling-based algorithms achieve asymptotic optimality by propagating random control inputs, their empirical convergence rate is often poor, especially in high-dimensional systems such as multirotors.
	An alternative approach is to first plan with a simplified geometric model and then use trajectory optimization to follow the reference path while accounting for the true dynamics.
	However, this approach
	may fail to produce a valid trajectory if the initial guess is not close to a dynamically feasible trajectory.
	In this paper, we present Iterative Discontinuity Bounded A* (iDb-A*), a novel kinodynamic motion planner that combines search and optimization iteratively.
	The search step utilizes a finite set of short trajectories (motion primitives) that are interconnected while allowing for a bounded discontinuity between them.
	The optimization step locally repairs the discontinuities with trajectory optimization.
	By progressively reducing the allowed discontinuity and incorporating more motion primitives, our algorithm achieves asymptotic optimality with excellent any-time performance.
	We provide a benchmark of 43 problems across eight different dynamical systems, including different versions of unicycles and multirotors.
	Compared to state-of-the-art methods, iDb-A* consistently solves more problem instances and finds lower-cost solutions more rapidly.
\end{abstract}


\section{Introduction}

\begin{figure}
	\centering
	\begin{subfigure}[t]{.23\textwidth}
		\centering
		\includegraphics[width=.9\linewidth]{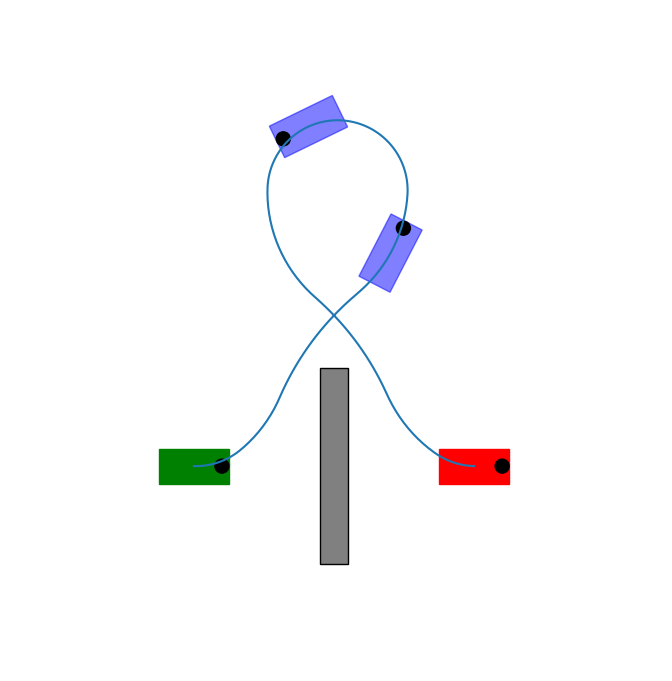}
		\caption{\label{fig:overview:uni1}}
	\end{subfigure}%
	\begin{subfigure}[t]{.23\textwidth}
		\centering
		\includegraphics[width=.9\linewidth]{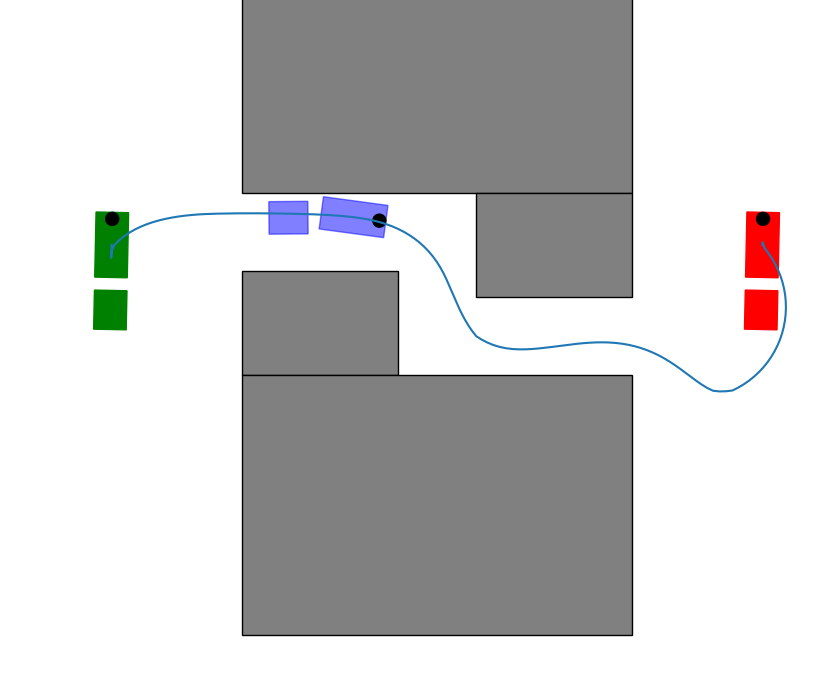}
		\caption{\label{fig:overview:car}}
	\end{subfigure}
	\begin{subfigure}[t]{.23\textwidth}
		\centering
		\includegraphics[width=.9\linewidth]{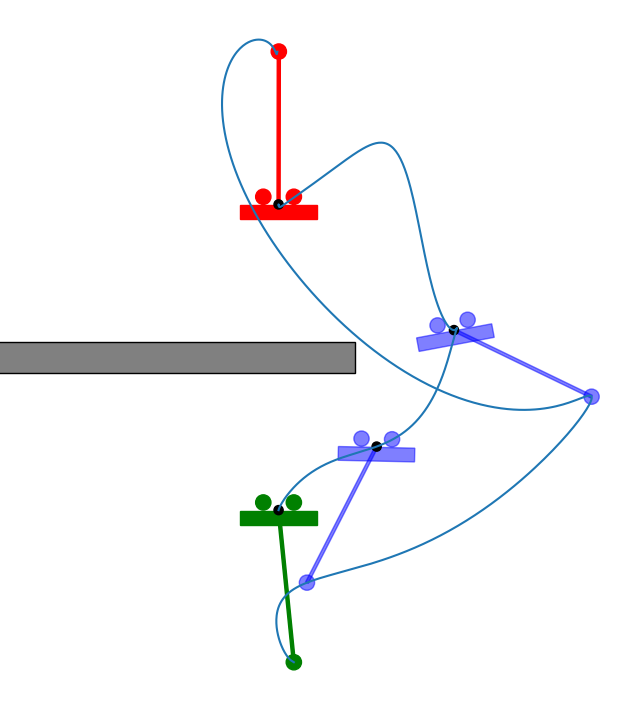}
		\caption{\label{fig:overview:pole}}
	\end{subfigure}%
	\begin{subfigure}[t]{.23\textwidth}
		\centering
		\includegraphics[width=.9\linewidth]{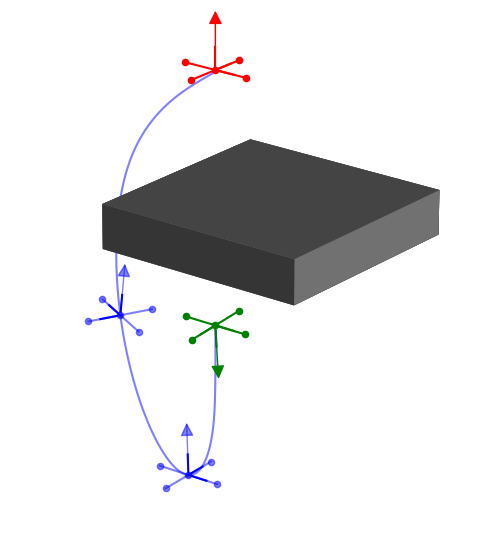}
		\caption{\label{fig:overview:3d_quad}}
	\end{subfigure}
	\caption{Examples of kinodynamic motion planning problems.
		Start and goal configurations are shown in green and red, respectively, while gray boxes represent obstacles.
		We display some trajectories found by our algorithm, iDb-A*, in blue.
		(a) A unicycle with asymmetric angular speed bounds and a positive minimum velocity. (b) A car pulling a trailer. (c) Acrobatics with a planar multirotor with an underactuated pendulum. (d) A recovery flight with a quadrotor with a very limited thrust-to-weight ratio.}
	\label{fig:intro}
\end{figure}

\IEEEPARstart{K}{inodynamic} motion planning for robots remains a challenging task, particularly when the objective is to compute time-optimal plans. \cref{fig:intro} showcases four interesting problems: obstacle-free recovery motions from unstable configurations with low-power quadcopters (\cref{fig:overview:3d_quad}), swing-up motions with acrobatic pole-copters among obstacles (\cref{fig:overview:pole}), maneuvering an Ackermann steering car with a trailer through a tight corridor (\cref{fig:overview:car}), and a unicycle with asymmetric angular speed bounds and a positive minimum velocity (\cref{fig:overview:uni1}).
Together, these examples illuminate the key challenges in kinodynamic motion planning:
\emph{i)} the diversity of the intended robotic systems, \emph{ii)} the nonlinearity of the dynamics, and \emph{iii)} nonconvex configuration spaces with narrow passages among obstacles.

Current planning approaches are sampling-based, search-based, optimization-based, or hybrid.
Each of these methods has its strengths and weaknesses.
Sampling-based planners \cite{SSTstar, DIRT, lavalle2001randomized}
can find initial solutions quickly and have strong guarantees for asymptotic convergence to an optimal solution in theory.
In practice, the initial solutions are far from optimal; the convergence rate is low, and the solutions typically require post-processing.

Search-based approaches \cite{pivtoraiko2011kinodynamic, likhachev2009planning} can remedy some of those shortcomings by connecting precomputed trajectories, so-called \emph{motion primitives}, using A* or related graph search algorithms.
Yet, the theoretical guarantees of A* only hold to the selected discretization of the state space and the precomputed motions.
Moreover, scaling this approach to higher dimensions or general systems requires careful, frequently hand-crafted design of the motion primitives, which requires domain-specific knowledge and is impractical for many dynamical systems.

Optimization-based planners \cite{TrajOpt, GuSTO, KOMO, li2004iterative} scale polynomially rather than exponentially with the number of state dimensions, which makes them better suited for high-dimensional planning problems.
However, these planners are, in general, only locally optimal and thus require a good initial guess both for the trajectory and the time horizon.

Thus, optimization approaches are typically combined with a sampling-based planner that generates an initial path with a simplified version of the dynamics, e.g., a geometric planner that avoids obstacles.
This strategy is not guaranteed to produce valid motion plans and requires in-depth knowledge of the dynamical system to choose an informative, yet simple enough, dynamics model.
For instance, a simple linear position model can be used to plan quadcopter motions around the hovering state but fails to plan trajectories that recover from upside-down configurations.

Our main contribution is \underline{I}terative \underline{D}iscontinuity \underline{B}ounded \underline{A*} (iDb-A*), a novel kinodynamic motion planner that combines a search algorithm, \underline{D}iscontinuity-\underline{B}ounded \underline{A*} (Db-A*), and trajectory optimization in an \underline{I}terative fashion.

iDb-A* combines key ideas and strengths of the previously introduced methods.
We rely on a graph search with short trajectories that are connected with bounded discontinuity because it provides a theoretically grounded exploration-exploitation trade-off.
We avoid predefined discretization and instead use a set of randomized motions similar to sampling-based planning.

Introducing discontinuity when connecting primitives makes the search tractable: we can reuse the primitives and have a finite number of states to expand.
While the output trajectory of the search algorithm is not feasible, it can be used as an initial guess for trajectory optimization that locally repairs the discontinuous trajectory into a valid one.
We execute search and optimization iteratively, where the value of the discontinuity bound decreases with each iteration, and the number of primitives increases.
For large discontinuity bounds, the search is fast, but the optimizer might fail to find a valid solution.
For small bounds, the search requires a longer runtime, but the optimizer has a better initial guess.
Thus, the iterative combination results in an efficient anytime planner with probabilistic optimality guarantees.

Our algorithm is implemented in C++ and is publicly available.
Our second contribution is an open-source benchmark that compares the three major kinodynamic motion planning techniques on the same problem instances.
While we focus in our evaluation on time optimality, our approach supports other cost functions that are additive and non-negative (e.g., energy or squared acceleration).

\textbf{Statement of Extension:}
This article is based on our previous conference paper \cite{hoenig2022dbAstar}, but it provides algorithmic improvements, a faster implementation, and a more extensive evaluation, which includes several problems that require obstacle avoidance and aggressive movements with flying robots.

\emph{New Algorithmic Contributions:}
\begin{itemize}
	\item A new strategy to optimize trajectories with free terminal time in the optimization step of iDb-A*.
	\item The generalization of iDb-A* from translation-invariant systems to systems without invariance (e.g., the acrobot), or with additional linear velocity invariance (e.g., multirotors).
\end{itemize}

Additionally, we provide a refined theoretical analysis and an ablation study of different components, including various optimization strategies, heuristics, and motion primitives.

\section{Related Work}

\textbf{Search-based} approaches rely on existing methods for discrete path planning, such as A* and its variants.
The common approach is to generate short trajectories (\emph{motion primitives}) using a state lattice—a pre-specified discrete set of states~\cite{pivtoraikoKinodynamicMotionPlanning2011, pivtoraiko2005efficient}.
Each primitive starts and ends at a grid cell, and swept cells can be precomputed for efficient collision checking.
Once motion primitives are computed, existing algorithms such as A* or an anytime variant (e.g., Anytime Repairing A* \cite{likhachev2003ara}) can be employed without modification, providing very strong theoretical guarantees on both optimality and completeness with respect to the chosen primitives.
The major challenge is selecting and computing effective motion primitives, especially for high-dimensional systems~\cite{PivtoraikoThesis, dispersionMinimizingPrimitives}.

\textbf{Sampling-based} approaches build a tree \(\mathcal{T}\) rooted at the start state \(\mathbf{x}_s\).
During tree expansion, i) a random state \(\mathbf{x}_{\text{rand}}\) in the state space is sampled, ii) an existing state \(\mathbf{x}_{\text{expand}} \in \mathcal{T}\) is selected, and iii) a new state \(\mathbf{x}_{\text{new}}\) is added with a motion that starts at \(\mathbf{x}_{\text{expand}}\) and moves towards \(\mathbf{x}_{\text{rand}}\).
The motions are typically generated by propagating random control inputs, and the classic version of this approach, \emph{kinodynamic RRT}~\cite{kinodynamicRRT}, is probabilistically complete~\cite{kunzKinodynamicRRTsFixed2015}.
Asymptotic optimality can be achieved when planning in state-cost space (\emph{AO-RRT})~\cite{AO-RRT, AO-RRT-Analysis, ST-RRT-Star} or by computing a sparse tree (\emph{SST*})~\cite{SSTstar}.
These methods rely on a distance function and often require fast nearest neighbor data structures, such as k-d trees, for efficiency.

Sampling-based approaches are designed to explore the state space as rapidly as possible and typically do not explicitly use a heuristic function, unlike search-based methods.
Instead, the exploration/exploitation trade-off is controlled by using goal-biasing.
In these approaches, the goal constraint is typically reformulated using a goal region rather than a goal state.

\textbf{Optimization-based} approaches locally refine an initial trajectory using the gradients of the cost function, dynamics, and collision constraints, unlike the previous gradient-free methods.
The trajectory optimization problem can be formulated as a finite-dimensional nonlinear program (NLP) using either direct collocation or shooting methods~\cite{betts2010practical}, and solved with general-purpose nonlinear solvers (e.g., \cite{wachter2006implementation, gill2005snopt}).

For instance, \emph{TrajOpt}~\cite{TrajOpt} and \emph{GuSTO}~\cite{GuSTO,malyutaConvexOptimizationTrajectory2021} rely on direct transcription and sequential convex programming (SCP), while \emph{KOMO}~\cite{KOMO} combines direct transcription with the Augmented Lagrangian algorithm.

Trajectories can also be computed with optimal control solvers based on Differential Dynamic Programming~\cite{Crocoddyl, howell2019altro} or the iterative Linear Quadratic Regulator (iLQR)~\cite{li2004iterative}.

All optimization-based approaches require an initial guess as a starting trajectory, but this guess does not necessarily need to be feasible.
For nonlinear dynamics and constraints, optimization approaches are incomplete and might fail or converge to a local optimum.
In fact, they often converge to infeasible solutions unless the initial guess is close to a feasible solution.
When successful, the solution quality is significantly higher (e.g., in terms of smoothness) compared to sampling-based or search-based approaches.
Moreover, optimization-based approaches do not suffer directly from the curse of dimensionality, although higher dimensions might result in more local optima.

\textbf{Hybrid} approaches combine search, sampling, and optimization.
For instance, one can combine search and optimization~\cite{natarajanInterleavingGraphSearch2021}, search and sampling~\cite{sakcakSamplingbasedOptimalKinodynamic2019, DIRT, shomeAsymptoticallyOptimalKinodynamic2021b}, or combine sampling and optimization~\cite{RABITstar, kamat2022bitkomo}.
For some dynamical systems, using insights from control theory for motion planning can also be beneficial \cite{LQR-RRTstar, kino-RRTstar, R3T}, but it requires domain knowledge.
Motion planning can also benefit from using machine learning for computational efficiency~\cite{RL-RRT, L-SBMP}.

Our algorithm, iDb-A*, combines ideas and tools from the three main approaches to kinodynamic motion planning.
The most closely related works are methods that reuse edges within a sampling-based planning framework~\cite{shomeAsymptoticallyOptimalKinodynamic2021b} and search-based methods with duplicate detection~\cite{duEscapingLocalMinima2019}.
Compared to these works, we include trajectory optimization and reuse locally optimal precomputed motion primitives interconnected with bounded discontinuity for better success and faster convergence.

Apart from the aforementioned approaches, a popular approach to kinodynamic motion planning problems is to first plan with \textbf{simplified dynamic models} and to use trajectory optimization or a local controller to follow the reference path while accounting for the true dynamics.
The simplest model is a geometric model (holonomic, first-order integrator), which enables geometric motion planning with, e.g., RRT, RRT*, PRM, or PRM* \cite{karaman2011sampling, kavraki1996probabilistic, lavalle1998rapidly}.
Second-order systems can be approximated by a double integrator linear model \cite{webb2012kinodynamic}.

The trajectories computed with simplified dynamics can then be used as initial guesses for trajectory optimization (that is, optimization-based approaches as previously discussed), model predictive control, or system-specific controllers for quadcopters \cite{lee2010geometric}, unicycle-like robots~\cite{kim2002controllers}, or car-like robots~\cite{laumond1998robot}.
System-specific motion planners can exploit certain properties of the dynamics, such as differential flatness in quadcopters, which allows faster motion planning as shown for quadcopters~\cite{mellinger2011minimum} and for some specific fixed-wing UAVs~\cite{bry2015aggressive}.
However, differential flatness cannot account for actuation constraints directly—leading to either conservative or infeasible trajectories, especially for small UAVs with a low thrust-to-weight ratio.

Notably, planning with simplified dynamics does not guarantee the generation of valid motion plans and demands an in-depth understanding of the dynamical system.
In our algorithm, iDb-A*, connecting motion primitives with bounded discontinuity during the search step can be interpreted as an alternative form of simplified dynamics.
However, iDb-A* is complete and asymptotically optimal because it combines search and optimization in an iterative fashion, increasing the number of motion primitives and reducing the allowed discontinuity in each iteration.

\section{Problem Description}
\label{sec:problem_description}

We consider a robot with a continuous state \(\vx \in \sX\) (e.g., \(\sX \subseteq \mathbb{R}^{d_x}\)) that is actuated by actions \(\vu \in \sU \subset \mathbb{R}^{d_u}\).
The dynamics of the robot are deterministic, described by a differential equation,
\begin{equation}
	\dot{\vx} = \vf(\vx, \vu).
	\label{eq:dynamics}
\end{equation}
To employ gradient-based optimization, we assume that we can compute the Jacobian of \(\vf\) with respect to \(\vx\) and \(\vu\), typically available in systems studied in kinodynamic motion planning, such as mobile robots or rigid-body articulated systems.
We use \(\mathcal{X}_{\text{free}} \subseteq \mathcal{X}\) to denote the collision-free space, i.e., the subset of states that are not in collision with the obstacles in the environment.

We discretize the dynamics \eqref{eq:dynamics} with a zero-order hold, i.e., we assume the applied action is constant during a time step of duration \(\Delta t\).
The discretized dynamics can then be written as,
\begin{equation}
	\label{eq:dynamics_discrete}
	\vx_{k+1} \approx \text{step}(\vx_k, \vu_k) \equiv \vx_k + \vf(\vx_k, \vu_k)\Delta t \,,
\end{equation}
using a small \(\Delta t\) to ensure the accuracy of the Euler approximation.
We use \(K \in \mathbb{N}\) to denote the number of time steps (which is not fixed but subject to optimization), \(\seqX = \langle \vx_0, \vx_1, \ldots, \vx_K \rangle\) to denote the sequence of states sampled at times \(0, \Delta t, \dots, K\Delta t\) and \(\seqU = \langle \vu_0, \vu_1, \ldots, \vu_{K-1} \rangle\) to denote the sequence of actions applied to the system for the time frames \([0,\Delta t), [\Delta t, 2\Delta t), \ldots, [(K-1)\Delta t, K\Delta t)\).
The objective of navigating the robot from its start state \(\vx_s\) to a goal state \(\vx_g\) can then be framed as the optimization problem,
\begin{subequations}
	\label{eq:motion-planning}
	\begin{align}
		  & \min_{\seqU,\seqX,K} J(\seqU,\seqX) \,, \label{eq:j}                                                           \\
		\text{s.t.
		} & \vx_{k+1} = \text{step}(\vx_k, \vu_k)                     & \forall k \in \{0,\ldots,K-1\} \label{eq:step} \,, \\
		  & \vu_k \in \sU                                             & \forall k \in \{0,\ldots,K-1\} \label{eq:u} \,,    \\
		  & \vx_k \in \sX_{\text{free}} \subseteq \sX                 & \forall k \in \{0,\ldots,K\} \label{eq:x} \,,      \\
		  & \vx_0 = \vx_s; \,\, \vx_K = \vx_g \label{eq:terminal} \,,
	\end{align}
\end{subequations}
with the cost term \(J(\mathbf{U},\mathbf{X}) = \sum_{k=0}^{K-1} j(\vu_k,\vx_k)\, \Delta t\), where \(j(\vu_k,\vx_k) \geq 0\).
In this paper, we will focus on time-optimal trajectories, i.e., \(j(\vu_k,\vx_k) = 1\); \(J(\seqU,\seqX, K) = K\Delta t\), but our framework can be applied to optimize any additive cost function, for example, minimum control effort \(j(\vu_k,\vx_k) = \|\vu_k\|^2\).

We assume the dynamics function \(\text{step}(\vx,\vu)\), control space \(\mathcal{U}\), state space \(\mathcal{X}\), and cost function \(j(\vx,\vu)\),
are known before solving the problem, which allows us to precompute motion primitives.

\begin{example}
	\label{ex:unicycle}
	Consider a unicycle robot with state \(\vx = [x, y, \theta] \in \mathbb{R}^2 \times SO(2)\), i.e., \(x, y\) are the position and \(\theta\) is the orientation.
	The actions are \(\vu = [v, \omega] \in \sU \subset \mathbb{R}^{2}\), i.e., the speed and angular velocity can be controlled directly.
	The dynamics are \(\dot{\vx} = [v \cos(\theta), v \sin(\theta), \omega]\).
	The choice of \(\sU\) can make this low-dimensional problem challenging to solve.
	For example, \cref{fig:overview:uni1} shows a plane-like case (positive minimum speed, i.e., \(0.25 \leq v \leq 0.5\) \SI{}{m/s}) with a malfunctioning rudder (asymmetric angular speed, i.e., \(-0.25 \leq \omega \leq 0.5\) \SI{}{rad/s}).
\end{example}

\begin{example}
	\label{ex:quadrotor}
	Consider a quadrotor \(\vx = [\mathbf{p}, \mathbf{v}, \mathbf{q}, \mathbf{w}]\) in \(\mathbb{R}^{9} \times SO(3)\) where \(\mathbf{p}\) represents the position, \(\mathbf{v}\) is the velocity, \(\mathbf{q}\) represents the orientation using a quaternion, and \(\mathbf{w}\) is the angular velocity in the body frame.
	The control input is the force at each rotor, \(\mathbf{u} \in \mathbb{R}^4\).
	The dynamics are,
	\begin{subequations}
		\begin{align}
			\mathbf{\dot{v}} = m^{-1} \mathbf{R}(\mathbf{q}) \mathbf{B}_1 \mathbf{u} + \mathbf{g},        \\
			\mathbf{\dot{w}} = \mathbf{I}^{-1} (\mathbf{B}_0 \mathbf{u} - \mathbf{w} \times \mathbf{Iw}), \\
			\mathbf{\dot{p}} = \mathbf{v}, \quad \mathbf{\dot{q}} = \frac{1}{2} \mathbf{q} \otimes \mathbf{w},
		\end{align}
	\end{subequations}
	where \(m\) is the mass, \(\mathbf{I}\) represents the inertia matrix, \(\mathbf{g}\) is the gravity vector, \(\mathbf{R}(\mathbf{q})\) is the rotation matrix corresponding to the quaternion \(\mathbf{q}\), and \(\otimes\) denotes the quaternion product.
	The matrices \(\mathbf{B}_0,\mathbf{B}_1 \in \mathbb{R}^{3\times4}\) are constant and depend on the quadcopter's geometry.
	The parameters of the Bitcraze Crazyflie 2.1 robot are used, which, with a very low thrust-to-weight ratio of 1.3 (i.e., \(0 \leq u_i \leq 1.3 \times g \times m /4\)), pose significant challenges for kinodynamic motion planning.
\end{example}

\begin{figure*}[ht]
	\centering
	\begin{tabular}{cccc}
		\includegraphics[width=.19\textwidth]{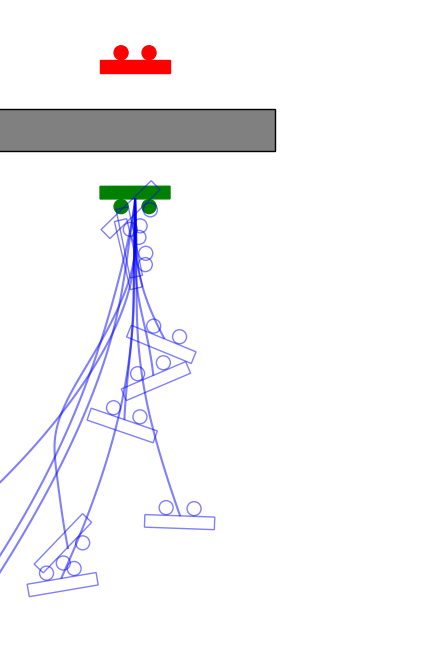}              &
		\includegraphics[width=.19\textwidth]{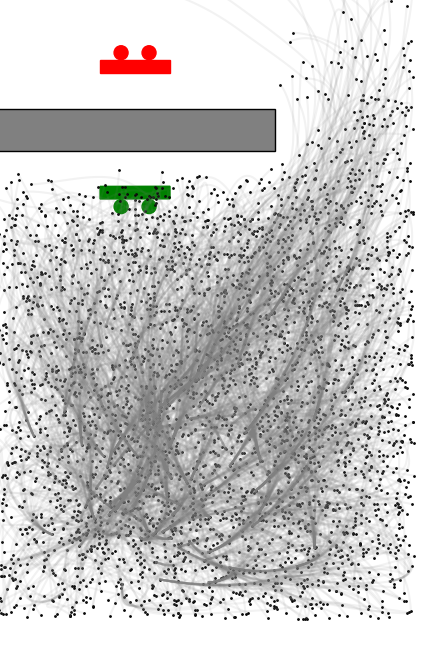} &
		\includegraphics[width=.19\textwidth]{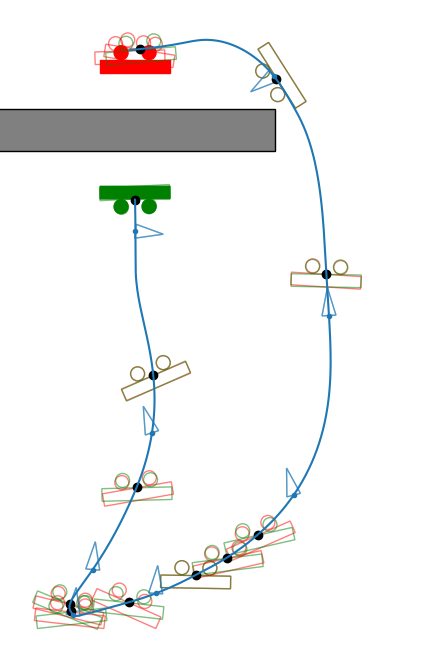}             &
		\includegraphics[width=.19\textwidth]{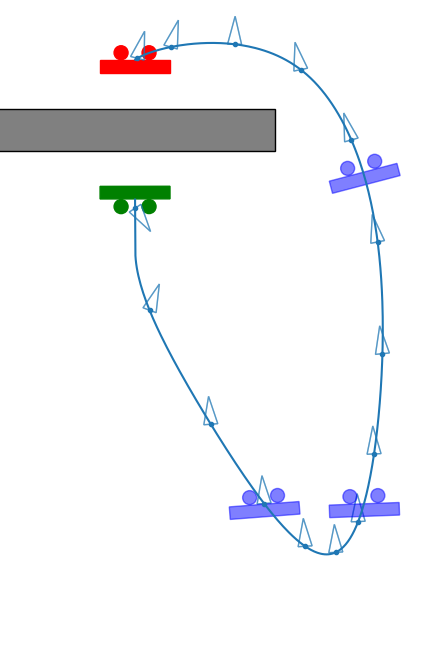}                                  \\

		(a)                                                                                                 & (b) & (c) & (d)
	\end{tabular}

	\caption{
		Visual representation of iDb-A* in the problem \emph{Planar rotor -- Recovery obstacles} (i.e., a recovery maneuver with a planar multirotor).
		Start and goal configurations are shown in solid green (note that the multirotor starts upside down) and red, respectively.
		(a) In the search step of iDb-A*, called Db-A*, we expand states (in this case, the initial state) using motion primitives that are applicable with bounded discontinuity.
		(b) Intermediate search tree during the execution of Db-A*.
		For visualization, the 6D configuration space is projected into a 2D space.
		(c) Solution found by Db-A*.
		The transparent, border-only green and red shapes show the start and end of each motion primitive, respectively.
		They do not match exactly, which highlights the allowed discontinuities when stitching motion primitives (note that the discontinuities in the velocities are not shown in this 2D representation).
		(d) The output of Db-A* is used to warm-start nonlinear trajectory optimization.
		The resulting trajectory, shown in blue, fulfills the dynamics constraints and is locally optimal.
	}
	\label{fig:visual_representation}
\end{figure*}

\section{iDb-A* - Overview}

Our iterative approach, which combines search and optimization, is detailed in \cref{alg:overview}.
We require a \emph{large} set of \emph{motion primitives} $\mathcal{M}_{L}$, which will be used incrementally and can be computed offline.
Motion primitives are short trajectories that fulfill our dynamics (see
\cref{def:motion-primitive} and \cref{sec:motion-primitives} for a formal definition and details on primitive generation).
In every iteration of iDb-A*, the following steps are performed:

\begin{enumerate}
	\item We increase the number of available motion primitives for the search (by choosing
	      new primitives from $\mathcal{M}_{L}$) and decrease the allowed discontinuity bound $\delta$ (\cref{alg:overview:sM,alg:overview:delta}).
	\item The discrete planner, Db-A*, computes a trajectory using the current set of motion primitives.
	      This trajectory may include a bounded violation of the dynamic constraints (\texttt{Db-A*} in \cref{alg:overview:dbAstar}, see \cref{sec:Discontinuity-Bounded-Search}).
	\item The result of Db-A* is used to initialize an optimization-based motion planner that attempts to compute a feasible and locally optimal trajectory (\cref{alg:overview:opt}).
	      See \texttt{Optimization} (\cref{sec:optimization}).
	\item Additional motion primitives are extracted from the output of the trajectory optimization (\texttt{Extract Primitives} in \cref{alg:overview:extract}).
\end{enumerate}

iDb-A* executes a sequence of A*-searches using a growing, randomized set of motion primitives, akin to Batch Informed Tree (BIT*) \cite{gammell2015batch}, a successful sampling-based planner for geometric motion planning.
In each iteration, the computation time and the success of the search and optimization steps depend on the number of motion primitives $n_i = |\mathcal{M}_i|$ and the allowed discontinuity bound $\delta_i$.
We can choose \texttt{AddPrimitives} (\cref{alg:overview:sM}) and \texttt{DecreaseDelta} (\cref{alg:overview:delta}) so that these parameters follow geometric sequences $n_{i+1} = n_{i} n_r$ and $\delta_{i+1} = \delta_i \delta_r$, with $n_r > 1$, $\delta_r < 1$ and initial values $n_0$, $\delta_0$.
We find this strategy easier to tune than the alternative approach presented in our prior work~\cite{hoenig2022dbAstar}, where we choose only the scheduling for the number of motion primitives and estimate the allowed discontinuity bound based on a desired approximate branching factor.

Motion primitives can also be extracted online in \cref{alg:overview}.
The \texttt{ExtractPrimitives} procedure utilizes the output of the optimization by dividing the trajectory into small sections.
The resulting primitives can be particularly useful for the planning problem at hand as they are computed with full knowledge of the environment.

A visual representation of some key components of iDb-A* is shown in \cref{fig:visual_representation}
using the problem \emph{Planar rotor -- Recovery obstacles}.

\begin{algorithm}[t]
	\caption{iDb-A* -- Iterative Discontinuity Bounded A*}
	\KwData{$\vx_s, \vx_g, \mathrm{step} , \sX_{\mathrm{free}}, \sU, \sM_{\mathrm{L}} $}
	\KwResult{$\seqX, \seqU$ }
	\label{alg:overview}
	\DontPrintSemicolon
	\SetKwFunction{AddPrimitives}{AddPrimitives}
	\SetKwFunction{ExtractPrimitives}{ExtractPrimitives}
	\SetKwFunction{ComputeDelta}{ComputeDelta}
	\SetKwFunction{ChooseDelta}{DecreaseDelta}
	\SetKwFunction{ChoosePrimitives}{IncreasePrimitives}
	\SetKwFunction{DiscontinuityBoundedAstar}{Db-A*}
	\SetKwFunction{Optimization}{Optimization}
	\SetKwFunction{Report}{Report}
	$\sM_0 \leftarrow \emptyset$ \Comment*{Initial Set of motion primitives}
	$c_{\mathrm{max}} \leftarrow \infty$ \Comment*{Solution cost bound}
	\For{$i=1,2,\ldots$}{
		$\sM_i \leftarrow \sM_{i-1} \cup \AddPrimitives(\mathcal{M}_L, i)$\label{alg:overview:sM}\;
		$\delta_i \leftarrow \ChooseDelta(i)$\label{alg:overview:delta}\;
		$\seqX_d, \seqU_d  \leftarrow$ \DiscontinuityBoundedAstar{$\vx_s, \vx_g, \sX_{\mathrm{free}}, \sM_i, \delta_i, c_{\mathrm{max}}$}\label{alg:overview:dbAstar}\;
		\If{$\seqX_d, \seqU_d$ successfully computed}{
			$\seqX, \seqU \leftarrow$ \Optimization{$\seqX_d, \seqU_d, \vx_s, \vx_g , \mathrm{step} , \sX_{\mathrm{free}} , \sU $}\label{alg:overview:opt}\;
			\If{$\seqX, \seqU$ successfully computed}{
				\Report{$\seqX, \seqU$} \Comment*{New solution found}
				$c_{\mathrm{max}} \leftarrow \min(c_{\mathrm{max}}, J(\seqX, \seqU))$ \Comment*{Cost bound}
			}
			$\sM_{i} \leftarrow  \sM_i \cup  \ExtractPrimitives(\seqX, \seqU)\label{alg:overview:extract}$\;
		}
	}
\end{algorithm}

\newpage

\section{Discontinuity Bounded A* Search}
\label{sec:Discontinuity-Bounded-Search}

\begin{figure*}[ht]
	\centering
	\begin{subfigure}[b]{.2\textwidth}
		\begin{tikzpicture}[dot/.style={inner sep=1pt, fill, circle}]
			\node[dot,color=green!40!gray!] (A) at (0,0) {};
			\node[dot, color= red!40!gray!] (G) at (\theGx, \theGy) {};

			\draw[white] (G) circle (0.6cm);
			\node[dot,color=white] (Gdelta) at (\theGdeltax,\theGdeltay) {};
			\node[dot, color=gray] (A1) at (0,0.3) {};
			\node[dot, color=gray] (A2) at (-0.2,0.1) {};
			\node[dot, color=gray] (A3) at (0.2,-.2) {};

			\node[dot,color=white] (B1) at (2,1) {};
			\node[dot] (B2) at (2,0.1) {};
			\node[dot,color=white] (B2d) at (2.05,0.3) {};
			\node[dot] (B3) at (1,-1) {};
			\node[dot,color=white] (C3) at (-.6,-.7) {};

			\draw[white] (C3) circle (0.4cm);

			\draw[gray] (A) circle (0.4cm);
			\node at (A) [below ] {\color{OliveGreen}$\vx_s$};
			\node at (G) [below] {\color{Red}$\vx_g$};
			\draw [\tHead,  dashed, thick , black ] (A1) .. controls (1,1) .. (B1);
			\draw [\tHead, thick, black] (A2) .. controls (1,-.1) ..  (B2);
			\draw [\tHead, thick, black] (A3) .. controls (.5,-1) .. (B3);
			\draw [\tHead, thick, gray] (.5,.5) .. controls (1.5,.2) .. (2.5,0.5);
			\draw [\tHead, thick, gray] (-.5,.5) .. controls (-.5, 1) .. (-.6,1.5);
			\draw [\tHead, thick, gray] (0, -1.2) .. controls (2.1, -1.5) ..  (2,-1);
			\draw [\tHead, thick, gray] (1, -.7) .. controls( 1.0, -.1) .. (2,-.1);
			\draw [\tHead, thick, gray] (.95, -.9) .. controls (.5, -.76) .. (C3);
			\draw [\tHead, thick, gray] (B2d) .. controls (3,-0.1) .. (Gdelta);
			\tObs
		\end{tikzpicture}
		\caption{Db-A* -- 1}
	\end{subfigure} \hspace{5mm}
	\begin{subfigure}[b]{.2\textwidth}
		\begin{tikzpicture}[dot/.style={inner sep=1pt, fill, circle}]
			\node[dot, color=green!40!gray!] (A) at (0,0) {};
			\node[dot, color= red!40!gray!] (G) at (\theGx,\theGy) {};
			\draw[white] (G) circle (0.6cm);
			\node[dot, color=gray] (A2) at (-0.2,0.1) {};
			\node[dot, color=gray] (A3) at (0.2,-.2) {};

			\node[dot,color=white] (B1) at (2,1) {};
			\node[dot] (B2) at (2,0.1) {};
			\node[dot, color=white] (B2d) at (2.05,0.3) {};
			\node[dot] (B3) at (1,-1) {};

			\draw[gray] (B2) circle (0.4cm);
			\draw[gray] (B3) circle (0.4cm);

			\node[dot] (C3) at (-.6,-.7) {};
			\draw[gray] (C3) circle (0.4cm);

			\draw[gray] (A) circle (0.4cm);
			\node at (A) [below ] {\color{OliveGreen}$\vx_s$};
			\node at (G) [below] {\color{Red}$\vx_g$};
			\draw [\tHead, thick, black] (A2) .. controls (1,-.1) ..  (B2);
			\draw [\tHead, thick, black] (A3) .. controls (.5,-1) .. (B3);
			\draw [\tHead, thick, gray] (.5,.5) .. controls (1.5,.2) .. (2.5,0.5);
			\draw [\tHead, thick, gray] (-.5,.5) .. controls (-.5, 1) .. (-.6,1.5);
			\draw [\tHead, thick, gray] (0, -1.2) .. controls (2.1, -1.5) ..  (2,-1);
			\draw [\tHead, thick, dashed, black] (1, -.7) .. controls( 1.0, -.1) .. (2,-.1);
			\draw [\tHead, thick, black] (.95, -.9) .. controls (.5, -.76) .. (C3);
			\draw [\tHead, thick, gray] (B2d) .. controls (3,-0.1) .. (Gdelta);

			\draw [dotted, thick, gray] (B3) -- (G);

			\tObs
		\end{tikzpicture}
		\caption{Db-A* -- 2}
	\end{subfigure} \hspace{5mm}
	\begin{subfigure}[b]{.2\textwidth}
		\begin{tikzpicture}[dot/.style={inner sep=1pt, fill, circle}]
			\node[dot, color=green!40!gray! ] (A) at (0,0) {};
			\node[dot, color= red!40!gray!] (G) at (\theGx,\theGy) {};
			\draw[gray] (G) circle (0.6cm);
			\node[dot, color=gray] (A2) at (-0.2,0.1) {};
			\node[dot, color=gray] (A3) at (0.2,-.2) {};

			\node[dot,color=white] (B1) at (2,1) {};
			\node[dot] (B2) at (2,0.1) {};
			\node[dot] (B2d) at (2.05,0.3) {};
			\node[dot] (B3) at (1,-1) {};

			\draw[gray] (B2) circle (0.4cm);
			\draw[gray] (B3) circle (0.4cm);

			\node[dot,color=black] (Gdelta) at (\theGdeltax,\theGdeltay) {};
			\node[dot] (C3) at (-.6,-.7) {};
			\draw[gray] (C3) circle (0.4cm);

			\draw[gray] (G) circle (0.6cm);
			\draw[gray] (A) circle (0.4cm);
			\node at (A) [below ] {\color{OliveGreen}$\vx_s$};
			\node at (G) [below ] {\color{Red}$\vx_g$};
			\draw [\tHead, thick, purple] (A2) .. controls (1,-.1) ..  (B2);
			\draw [\tHead, thick, black] (A3) .. controls (.5,-1) .. (B3);
			\draw [\tHead, thick, gray] (.5,.5) .. controls (1.5,.2) .. (2.5,0.5);
			\draw [\tHead, thick, gray] (-.5,.5) .. controls (-.5, 1) .. (-.6,1.5);
			\draw [\tHead, thick, gray] (0, -1.2) .. controls (2.1, -1.5) ..  (2,-1);
			\draw [\tHead, thick, black] (.95, -.9) .. controls (.5, -.76) .. (C3);
			\draw [\tHead, thick, purple] (B2d) .. controls (3,-0.1) .. (Gdelta);


			\tObs
		\end{tikzpicture} 
		\caption{Db-A* -- 3}
	\end{subfigure} \hspace{5mm}
	\begin{subfigure}[b]{.2\textwidth}
		\begin{tikzpicture}[dot/.style={inner sep=1pt, fill, circle}]
			\node[dot, color=green!40!gray! ] (A) at (0,0) {};
			\node[dot, color= red!40!gray!] (G) at (\theGx,\theGy) {};
			\node[dot, color=gray] (A2) at (-0.2,0.1) {};

			\node[dot] (Gdelta) at (\theGdeltax,\theGdeltay) {};

			\node[dot] (B2) at (2,0.1) {};
			\node[dot] (B2d) at (2.05,0.3) {};

			\draw[gray, white] (G) circle (0.6cm);


			\node at (A) [below ] {\color{OliveGreen}$\vx_s$};
			\node at (G) [below ] {\color{Red}$\vx_g$};
			\draw [\tHead, thick, purple] (A2) .. controls (1,-.1) ..  (B2);
			\draw [\tHead, thick, purple] (B2d) .. controls (3,-0.1) .. (Gdelta);

			\draw [\tHead, thick, blue!70!gray!] (A) .. controls (2.3,.2) ..  (G);


			\tObs
		\end{tikzpicture}
		\caption{Optimization}
	\end{subfigure}
	\caption{
		(a, b, c) A graphical description of Db-A*.
		The gray edges represent motion primitives, and the states
		\( \vx_s \) and \( \vx_g \) are the start and the goal, respectively.
		(a) Given an initial state \( \vx_s \), we can only apply primitives that start with a discontinuity lower than \( \alpha \delta \) (gray circumference) and are collision-free.
		The applicable primitives are shown with solid black edges.
		(b) The search is ordered by a heuristic (e.g., the \emph{Euclidean} heuristic, shown with a dotted line for the best node) and the cost-to-come.
		When expanding a node, we create new states only if they are not within \( (1-\alpha) \delta \) of a previously discovered state (i.e., the dashed edge is not expanded).
		(c) The search is terminated when a node close to the goal is expanded.
		(d) The solution of Db-A* will be used to warm-start a trajectory optimization algorithm that repairs the discontinuities and locally optimizes the trajectory.
		The optimized trajectory is shown in blue.
	} \label{fig:graphical-dbastar}
\end{figure*}
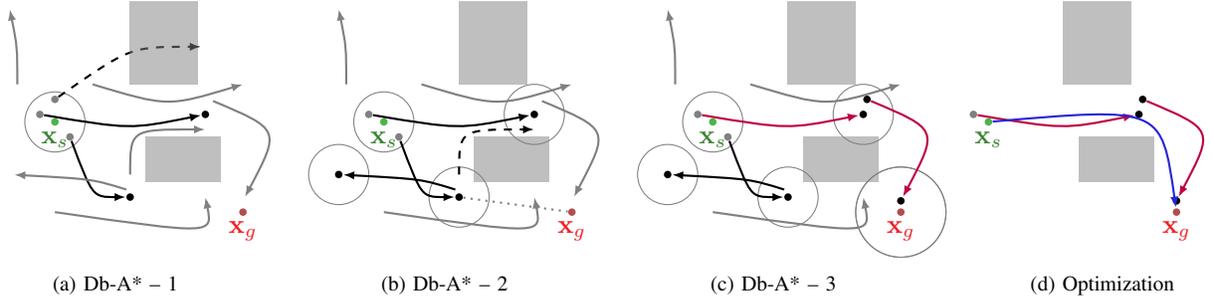

Discontinuity Bounded A* (Db-A*) is a search algorithm that uses a set of motion primitives, which are connected while allowing for a maximum discontinuity.

A motion primitive is a sequence of states and controls that fulfill the dynamics of the system.
Formally,

\begin{definition}[Motion Primitive]
	A motion primitive \( m = (\seqX, \seqU, \vx_s, \vx_f, c) \) is a sequence of states \( \seqX = (\vx_0, \ldots ,\vx_N) \), \( \vx_k \in \sX \), and controls \( \seqU = (\vu_0, \ldots, \vu_{N-1}) \), \( \vu_k \in \sU \) that fulfill the dynamics \( \vx_{k+1} = \step(\vx_k,\vu_k) \).
	It connects the start state \( \vx_s = \vx_0 \) and the final state \( \vx_f = \vx_N \), with a corresponding cost \( c \in \RR^+ \).
	The length of the motion primitive (i.e., the number of states and controls) is randomized.
	\label{def:motion-primitive}
\end{definition}

In the following, we rely on a user-specified \emph{metric} \( d: \sX \times \sX \to \mathbb{R}^+ \), which measures the distance between two states (e.g., a weighted Euclidean norm).
We assume that \( \langle \sX, d\rangle \) is a metric space in order to use efficient nearest-neighbor data structures, such as k-d trees.

\begin{definition}
	\label{definition:discontinuityBounded}
	The pair of sequences \( \seqX = \langle \vx_0, \ldots, \vx_K\rangle \), \( \seqU = \langle \vu_0, \ldots, \vu_{K-1}\rangle \) is a \( \delta \)-discontinuity bounded solution (with \( \delta>0 \)) to the kinodynamic motion planning problem \cref{eq:motion-planning} if and only if the following conditions hold:
	\begin{subequations}
		\begin{align}
			d(\vx_{k+1}, \step(\vx_k, \vu_k)) \leq \delta \quad \forall k \in \{0,\ldots, K-1\} \label{eq:disBound:dynamics}\,, \\
			\vu_k \in \sU \quad \forall k \in \{0,\ldots, K-1\}\,, \label{eq:disBound:U}                                        \\
			\vx_k \in \sX_{\mathrm{free}} \quad \forall k \in \{0,\ldots, K\}\,, \label{eq:disBound:Xfree}                      \\
			d(\vx_0, \vx_s) \leq \delta \,, \label{eq:disBound:xs}                                                              \\
			d(\vx_K, \vx_g) \leq \delta \,.
			\label{eq:disBound:xf}
		\end{align}
	\end{subequations}
\end{definition}

Intuitively, \cref{definition:discontinuityBounded} enforces that the sequences connect the start and goal states with a bounded error \( \delta \) in the dynamics, which corresponds to “stitching” primitives together.

\subsection{Algorithm}

Our approach to computing such sequences is Discontinuity Bounded A* (Db-A*).

Db-A*, like A*, is an informed search that relies on a heuristic $h: \sX \to \mathbb R$ to explore an implicitly defined directed graph efficiently.
Nodes in the graph represent states, and an edge between two nodes indicates that there exists a motion that connects the states, allowing up to $\delta$-discontinuity.

The algorithm is shown in \cref{alg:dbAstar}, and \cref{fig:graphical-dbastar} provides a graphical representation.
Db-A* keeps track of nodes to explore using a priority queue, which is sorted by the lowest $f(\vx)=g(\vx)+h(\vx)$ value, where $g(\vx)$ is the cost-to-come.
The overall structure is the same as in A*: The OPEN priority queue $\mathcal O$ is initialized with the start state (\cref{alg:dbAstar:Oinit}).
At each iteration, we remove the first element from $\mathcal O$ (\cref{alg:dbAstar:Opop}), and that node is expanded using the applicable collision-free motion primitives (\crefrange{alg:dbAstar:expand1}{alg:dbAstar:expand2}).
A motion $m$ is applicable in state $\vx$ if its start state $m.
	\vx_s$ is within a distance of at most $\alpha \delta$ (\cref{alg:dbAstar:Mprime}), resulting in a new state $m.\vx_f$.

New states are added to $\mathcal O$ (\cref{alg:dbAstar:Oadd}) only if they are not within $(1-\alpha)\delta$ of previously discovered nodes.
If the state is close to a previous node, the previous node is updated if the new cost-to-come is reduced (\cref{alg:dbAstar:update}).
Therefore, unlike A*, we consider two states to be equivalent if they are within $(1-\alpha)\delta$ of each other.

For computing and updating the cost to come, we consider the cost of the motion primitive $m.c$ and the cost of the discontinuity bound using a lower bound function $l: \sX \times \sX \to \mathbb R^+$ of the true cost.
Therefore, given a state $\vx$ with cost to come $g(\vx)$, the cost of a new state $\vx' = \vx \oplus m = m.
	\vx_f$ is $g(\vx') = g(\vx) + l(\vx, m.\vx_s) + m.c$ (\cref{alg:dbAstar:gt}).

The search terminates when we find a node that is within $\delta$ distance of the goal state (\cref{alg:dbAstar:sol}).

For efficient search, we employ two k-d trees.
The first tree indexes the start states of all provided motion primitives, which can be done once at the beginning.
The second k-d tree contains the states of all explored nodes and grows dynamically.
It is used to find nearby previously explored states.
The discontinuity with a magnitude of up to $\delta$ may occur in two cases: first, when we select suitable motion primitives for expansion (\cref{alg:dbAstar:Mprime}), and second, when we prune a potential new node in favor of already existing states (\cref{alg:dbAstar:nn}).
The tradeoff between the two can be adjusted by a user-specified parameter $\alpha \in (0, 1)$.

\begin{algorithm}[t]
	\caption{Db-A* -- Discontinuity Bounded A*}
	\label{alg:dbAstar}
	\DontPrintSemicolon

	\SetKwFunction{NearestNeighborInit}{NearestNeighborInit}
	\SetKwFunction{NearestNeighborQuery}{NearestNeighborQuery}

	\SetKwFunction{IsStateNovel}{IsStateNovel}

	\SetKwFunction{NearestNeighborAdd}{NearestNeighborAdd}
	\SetKwFunction{PriorityQueuePop}{PriorityQueuePop}
	\SetKwFunction{PriorityQueueInsert}{PriorityQueueInsert}
	\SetKwFunction{PriorityQueueUpdate}{PriorityQueueUpdate}
	\SetKwFunction{UpdateNode}{UpdateNode}

	\KwData{$\vx_s, \vx_g, \sX_{\mathrm{free}}, \sM, \delta, c_{\mathrm{max}}$}
	\KwResult{$\seqX_d, \seqU_d$ or Infeasible}

	$\mathcal O \leftarrow \{Node( \vx: \vx_s, g: 0, h: h(\vx_s), p: None, a: None) \}$ \label{alg:dbAstar:Oinit} \quad \Comment{Initialize open list (priority queue)}
	$ \mathcal C \leftarrow \{\} $ \quad \Comment*{Initialize list of closed nodes}
	\While{$|\mathcal O| > 0$}{
		\Comment{Remove node with lowest f-value} \label{alg:dbAstar:Opop}
		$n \leftarrow \PriorityQueuePop(\mathcal O)$

		\If{$d(n.\vx, \vx_g) \leq \delta $}  {
			\Return $\seqX_d, \seqU_d $ \Comment*{Trace back solution} \label{alg:dbAstar:sol}
		}

		\Comment{Find applicable motion primitives with discontinuity up to $\alpha \delta$}
		$\mathcal M' \leftarrow \NearestNeighborQuery(n.\vx , \mathcal M, \alpha\delta)$
		\label{alg:dbAstar:Mprime}\;

		\ForEach{$m\in \mathcal M'$ \label{alg:dbAstar:expand1}}{

			\If{ $ m  \notin \sX_{\mathrm{free}}$ \label{alg:dbAstar:collision}}{
				\Continue \label{alg:dbAstar:expand2} \Comment*{Motion is not collision-free}
			}

			$\vx' \gets m.\vx_f $ \Comment*{Tentative new state}

			$g' \leftarrow n.g + m.c + l(n.\vx , m.\vx_s) $ \Comment*{Tentative g score}
			\label{alg:dbAstar:gt}

			\If{ $ g' > c_{\mathrm{max}}$}{
				\Continue
			}

			\Comment{Check if we have previously discovered states within $(1-\alpha)\delta$}
			$\mathcal{N} = \NearestNeighborQuery(  \vx', \mathcal{O} \cup \mathcal{C}, (1-\alpha) \delta)$
			\label{ref:dbAstar:nn}\;
			\If{$ \mathcal{N} = \emptyset$ }{
				\PriorityQueueInsert$(\mathcal O, Node(\vx', g', h(\vx'), n, m))$
				\label{alg:dbAstar:Oadd}\;
			}
			\Else																														{

				\ForEach{$s\in \mathcal N$ \label{alg:dbAstar:nn}}{
					$g'' = g' + l(\vx',s.\vx)$\;
					\If{ $g'' < s.g$ } {
						$s = \UpdateNode(\{g:g'', p:n, a:m\})$\label{alg:dbAstar:update} \Comment*{Update node.
							If it is in closed list, reinsert in open list.
						}
					}
				}
			}

		}

		Insert $n$ in $\mathcal{C}$

	}
	\Return Infeasible
\end{algorithm}

\subsection{Heuristic Functions}
\label{sec:heuristic}

In kinodynamic motion planning, three heuristic functions \( h(\vx) \) are particularly relevant:

\paragraph{Euclidean Heuristic}
The \textit{Euclidean} heuristic is based on the Euclidean distance to the goal, considering state and control constraints such as maximum velocity or acceleration, while ignoring dynamics and obstacles.
It is usually computed by a combination of weighted Euclidean or infinity norms and does not require any precomputation.

\paragraph{Roadmap Heuristic}
The \textit{Roadmap} heuristic approximates the collision-free space using a geometric roadmap, thus taking collisions and control bounds into account but ignoring the dynamics.
It requires a precomputation step to build the geometric roadmap, which can be reused between iterations of iDb-A*, and it is usually more informative in problems where obstacles play a significant role.
Given a finite set of state-cost pairs \( S = \{(\vs_i,c_i) | ~ \vs_i \in \mathcal{X}, c_i \in \mathbb{R}\} \), the heuristic function is given by:
\begin{equation}
	h(\vx) = \min_{(\vs_i, c_i) \in S : ~ d(\vs_i, \vx) \leq R} \{c_i + l(\vs_i, \vx) \} \,,
\end{equation}
where \( l \) is a lower bound on the cost for reaching \( \vs_i \) from \( \vx \), and \( R \) is a user-defined connection radius.
To compute \( S \), we construct a roadmap with randomly sampled configurations and annotate each vertex with the geometric cost-to-go (i.e., using the \textit{Euclidean} heuristic for each collision-free edge).
Each query requires a nearest-neighbor search (implemented using a k-d tree).

\paragraph{Blind Heuristic}
Lastly, we also evaluate the \textit{Blind} heuristic, where \( h(\vx) = 0, ~ \forall \vx \).
This heuristic is motivated by systems where the dynamics play a central role, such that feasible trajectories of the robot strongly differ from straight lines in the state space (e.g., the acrobot), making the \textit{Euclidean} and \textit{Roadmap} heuristics uninformative.

\begin{example}
	Consider the unicycle robot, with state and dynamics as in Example \ref{ex:unicycle}.
	Given a state \( \vx=[x,y,\theta] \), goal \( \vg=[g_x,g_y,g_\theta] \), and the control bounds \( |v| \leq v_{\textup{max}} \), \( |w| \leq w_{\textup{max}} \), the \textit{Euclidean} heuristic is:
	\begin{align}
		h(\vx) = \max \{ ~ & v_{\textup{max}}^{-1} \norm{ [x, y] - [g_x, g_y] }, \nonumber \\  &
		   w_{\textup{max}}^{-1} D_{\theta}(x_{\theta}, g_{\theta}) \}\,,
	\end{align}
	where \( D_{\theta}( \cdot, \cdot) \) is the distance metric in SO(2).
\end{example}

\subsection{Equivalence Between Continuous States}

A fundamental issue when applying search algorithms in continuous spaces is that the set of possible reachable states is infinite.
The search algorithm will unnecessarily expand similar states, especially when the heuristic is not informative.

To mitigate this issue, in Db-A* we have defined a notion of similarity or equivalence between states, often referred to as duplicate detection in related work \cite{duEscapingLocalMinima2019, gonzalez2011search, maray2022improved}.
In \cref{alg:dbAstar}, a state is considered not novel if it is close to a previously discovered state, in which case the state is pruned.
This makes Db-A* incomplete and suboptimal for fixed values of \( \delta \).
As Db-A* runs for decreasing \( \delta \) inside iDb-A*, the size of the equivalence class is iteratively reduced.
In combination with subsequent optimization, we found our duplicate detection to be sufficient for both good practical performance and asymptotic optimality (\cref{sec:theory}).

\subsection{Invariance and Equivariance in the Motion Primitives}
\label{sec:inv-equi-search}
To decide which motion primitives are applicable in a state (\cref{alg:dbAstar:Mprime}), we can exploit invariance and equivariance in the system dynamics, which allows us to reuse the same primitive in different states with smaller discontinuities.

A prominent example is the translation invariance of the dynamics of mobile robots.
Intuitively, a valid motion primitive can be ``translated'' to match other starting states so that there is no discontinuity in the translation components of the state.
This concept is formalized in \cref{sec:motion-primitives}, where we provide two examples: translation invariance for a car-like robot and translation and linear velocity invariance for flying robots.

From an implementation perspective, to account for invariances, all primitives are stored in a canonical form (e.g., with 0 translation component) inside a k-d tree.
At runtime, we transform the query state into the canonical form to check which primitives are applicable, and the valid primitives are then transformed on-the-fly to expand the query state.

\subsection{Efficient Collision Checks with Collision Shapes}

Collision checking is one of the most expensive operations in motion planning.
To check collisions between the environment and a motion primitive (\cref{alg:dbAstar:collision}), we use either precomputed collision shapes of the motion primitive (if available) or check collisions at a small temporal resolution.
Importantly, precomputed collision shapes can also be transformed online for any translation or rotation of the motion primitive.
When available, collision shapes are considerably faster than checking individual configurations at a chosen resolution.

In practice, we observe that the running time of Db-A* is dominated by both nearest neighbor searches to find neighboring states and applicable motion primitives, and by collision detection (see \cref{fig:time-spent-each-component} in \cref{sec:ablation}).

\section{Trajectory Optimization}
\label{sec:optimization}

For the \texttt{Optimization} subroutine (\cref{alg:overview:opt} in
\cref{alg:overview}), we use gradient-based trajectory optimization.
We assume that the derivatives of the dynamics, the distance function, and the collision constraints can be computed efficiently, e.g., using analytical expressions, a differentiable simulator, or finite differences.
For collisions, we now require a signed distance function instead of a binary collision check.

The objective of the optimization is to solve the original kinodynamic motion planning problem \eqref{eq:motion-planning}, using the output of Db-A* as an initial guess, $\seqX_d = \{\vx_0 , \vx_1 , \ldots , \vx_{K} \} $, $\seqU_d = \{\vu_0 , \vu_1 , \ldots , \vu_{K-1} \} $.

Even with the initial guess from Db-A*, the optimization problem \eqref{eq:motion-planning} is challenging for gradient-based optimization, especially when starting with large discontinuities.

The problem is nonconvex even for systems with linear dynamics and constraints, and the infeasible initial guess and underactuation of the systems prevent the use of time-optimal path tracking approaches (e.g., \cite{verscheure2009time, pham2018new}).
In this section, we describe four different methods that we studied for the optimization step of iDb-A*, based on different approaches in the optimization and control literature.

\paragraph{Joint Optimization of Trajectory and Terminal Time (Free-dt)}

This approach adds the duration of the time interval as an optimization variable for joint nonlinear optimization of time and trajectory:
\begin{subequations}
	\label{eq:dt}
	\begin{align}
		\min_{\mathbf{X}, \mathbf{U} , \Delta t } & \sum_k \Delta t~ j(\vx_k,\vu_k) \,,                                             \\
		\text{s.t.
		} \quad                                   & \vx_{k+1} = \vx_k + \vf(\vx_{k},\vu_k) \Delta t \,,                             \\
		                                          & \text{Constraints} ~ \eqref{eq:u}, \eqref{eq:x}~\text{and}~\eqref{eq:terminal}.
	\end{align}
\end{subequations}
Here, $\Delta t $ is a variable, initialized to $\Delta t_{\text{ref}}$, the reference value used for time-discretization in the motion primitives for the given dynamical system.
The number of time steps $K$ is fixed.
After solving \eqref{eq:dt}, we would like to have the solution trajectory discretized with the original time step duration for consistency.
Thus, we recompute the state and control trajectories using the reference time step $\Delta t_{\text{ref}}$.
This requires i) interpolation of the solution of \eqref{eq:dt} with $\Delta t_{\text{ref}}$, and ii) a second run of trajectory optimization, now with fixed $\Delta t = \Delta t_{\text{ref}}$ to repair the small errors arising from the Euler integration with different step sizes (note that the second optimization is very efficient because the interpolation of the solution of \eqref{eq:dt} is already accurate).

\paragraph{Hierarchical Time Search (Search-T)}
\label{sec:hierarchical-time-search}

A hierarchical approach that combines a linear search on the terminal time with trajectory optimization with a fixed terminal time.
Given the time bounds $\{T_{\text{min}}, T_{\text{max}}\}$ and a time resolution $h$, we define a set of candidate times $\mathcal{T} = \{ T_\text{min}, T_\text{min} + h, \ldots, T_\text{max} \}$, and solve the hierarchical optimization problem:
\begin{equation}
	\label{eq:search-t}
	\min_{ T_i \in \mathcal{T} } ~ \text{Trajectory\_Optimization}(T_i),
\end{equation}
where $\text{Trajectory\_Optimization}(T_i)$ first rescales temporally the initial guess to have a time duration of $T_i$ (the time step size $\Delta t$ is kept constant, but the number of time steps of the trajectory varies) and then solves \eqref{eq:motion-planning} with a fixed number of time steps $K_i$ and fixed $\Delta t$.
For time-optimal trajectories, we start the search at $T_{\text{min}}$, and stop at the first $T_i$ when $\text{Trajectory\_Optimization}(T_i)$ is feasible.
Based on the duration $T_0$ of the initial guess, a reasonable choice of the parameters is, e.g., $ T_{\text{min}} = 0.5 T_0$, $T_{\text{max}} = 2 T_0$, and $h = (T_{\text{max}} - T_{\text{min}}) / 10$.

\subsection{Sliding Window Optimization}

Instead of considering the full trajectory at once, the optimization step can try to repair the discontinuities locally.
While this approach is more constrained to follow the initial guess, it is also potentially faster.
Solving a sequence of smaller subproblems often reduces the computational cost and number of nonlinear iterations, which typically increase for longer trajectories.

The following two approaches are inspired by two optimal control formulations, namely Model Predictive Control (MPC) and Model Predictive Contouring Control (MPCC) \cite{lam_mpcc}.
In both approaches, we repair the initial guess trajectory in a sequence of steps, starting from the beginning of the initial guess.
At each step, we (i) optimize the sequence of states and controls inside a small optimization window of length \(W_o\) (e.g., 50 steps), (ii) fix the first \(W_s \leq W_o\) states and controls (e.g., 10 steps), and (iii) move the optimization window by \(W_s\), so that the new start state is the last fixed state.
The time step duration \(\Delta t\) is fixed, but the resulting final trajectory might have a different duration than the initial guess from Db-A*.

\paragraph{Subgoal Following (MPC)}
The optimization problem in each step is:
\begin{subequations}
	\begin{align}
		\min_{ \seqX_W, \seqU_{W} } & k_1 d( \vx_W , \mathbf{g} )^2 + \Delta t \sum_w \, j(\vx_w,\vu_w) \,,         \\
		\text{s.t.
		} \quad                     & \text{Constraints} ~ \eqref{eq:step}, \eqref{eq:u},~ \text{and}~\eqref{eq:x}.
	\end{align}
\end{subequations}
Here, \(\seqX_W,\seqU_W\) are the sequence of states and controls in the optimization window, \(\vx_W\) is the last state of the current window, and \(\mathbf{g}\) is the subgoal state for this optimization window, chosen from the Db-A* initial guess to encourage making progress in the path.
The weight \(k_1 > 0\) combines the objective of minimizing the distance to the subgoal \(d(\vx_W, \mathbf{g})\) with the original control cost function.

\paragraph{Path Following (MPCC)}
The optimization problem in each step is:
\begin{subequations}
	\begin{align}
		\min_{ \seqX_W, \seqU_{W}, \alpha}        & - k_1 \alpha + k_2 d(\vx_W , \pi(\alpha))^2 + \nonumber                       \\
		                                          & \Delta t \sum_w \, j(\vx_w,\vu_w) \,,                                         \\
		\text{s.t.
		}                                   \quad & 0 \leq \alpha \leq 1 \,,                                                      \\
		                                          & \text{Constraints} ~ \eqref{eq:step}, \eqref{eq:u}, ~\text{and}~\eqref{eq:x}.
	\end{align}
\end{subequations}
The function \(\pi(\cdot): [0, 1] \to \mathcal{X}\) is a smooth parameterization of the initial guess (e.g., a spline through the waypoints), and the scalar variable \(\alpha\) indicates the progress on the path.
The term \(k_1 \alpha\), with \(k_1 > 0\), tries to maximize the progress along the path.
The term \(k_2 d(\vx_W , \pi(\alpha))^2\), with \(k_2 > 0\), minimizes the distance between the last state in the window and the progress on the path.
Together, these two terms push the last state \(\vx_W\) to make progress along the path while following it closely (note that, compared to other MPCC formulations, we only apply the contouring cost to the last state in the window).

\subsection{Algorithm for Trajectory Optimization}

All different approaches for trajectory optimization require solving nonlinear optimal control problems.
In our previous work \cite{hoenig2022dbAstar}, we used a direct control method, namely k-order optimization \cite{KOMO}, and the Augmented Lagrangian algorithm.
In this revised version, we switch to an indirect control method, Differential Dynamic Programming (DDP), which ensures precise dynamics during shooting and therefore more reliable convergence to locally optimal solutions in systems with complex dynamics.

Differential Dynamic Programming is a second-order method for solving optimal control problems of the form:
\begin{subequations}
	\begin{align}
		\min_{ \seqX, \seqU } & \sum_k c(\vx_k,\vu_k) + c_K(\vx_K)\,,                                          \\
		\text{s.t.
		} \quad               & \vx_{k+1} = \text{step}(\vx_{k},\vu_k) \quad \forall k \in \{0,\ldots,K-1\}\,, \\
		                      & \vx_0 = \vx_s\,.
	\end{align}
\end{subequations}
It iteratively computes a quadratic approximation of the cost-to-go using a backward pass and updates states and controls using a forward pass.
For more details, we refer to \cite{Crocoddyl}, \cite{ddp}.

To deal with collisions, goal constraints, and state and control bounds, we use a squared penalty method—adding all constraints in the cost term with a squared penalty.

In particular, we use \emph{feasibility-driven DDP} \cite{Crocoddyl}, which can be warm-started with an infeasible sequence of states and actions, providing a good balance between local convergence and globalization.

\section{Motion Primitives}
\label{sec:motion-primitives}

In our framework, we define a motion primitive as a valid trajectory that fulfills the dynamics, control, and state constraints, disregarding collisions with the environment.
A formal definition is provided in \cref{def:motion-primitive} in \cref{sec:Discontinuity-Bounded-Search}.

\subsection{Generation of Locally Optimal Motion Primitives}

In the problem setting outlined in \cref{sec:problem_description}, a key observation is that motion primitives can be precomputed because they are independent of the collision-free space and the start and goal configurations of a particular motion planning problem.

Our algorithms, iDb-A* and Db-A*, are agnostic about how the primitives have been generated.
However, the theoretical properties of the algorithms depend on the properties of the set of primitives (\cref{sec:theory}).
In our implementation of iDb-A*, we use locally optimal motion primitives computed with trajectory optimization.

To generate motion primitives, we solve two-point boundary value problems with random start and goal configurations in free space using nonlinear optimization.
In contrast to the typical approach in sampling-based motion planning of sampling control sequences at random, our strategy results in a superior primitive distribution, especially for systems with unstable dynamics (e.g., flying robots).
Our approach achieves better coverage of the state space and produces smoother and lower-cost motion primitives, which are key factors contributing to the success of the algorithm (see \cref{sec:experiments}).

Specifically, we generate motion primitives offline using the following three steps: First, uniform random sampling of start and goal configurations in free space; second, solving \cref{eq:motion-planning} with trajectory optimization using the hierarchical time search approach and a trivial initial guess; and third, splitting the resulting motion into multiple pieces of random length.

We observe that our strategy generates a good distribution of motion primitives.
However, it requires several hours of offline computation with a standard CPU.
In the case of flying robots, most of the time is spent attempting to find trajectories where the goal is not reachable within the given time horizon or where trajectory optimization fails to converge.
Our sampling approach might bias the primitives' distribution towards configurations that are easy to connect to other configurations.
However, the results in \cref{sec:experiments} confirm that the primitives are diverse and can solve a wide range of problems.

\subsection{Invariance and Equivariance in System Dynamics}
\label{sec:invariance}

Several robotic systems of interest exhibit symmetries and invariances in the dynamics that can be exploited to reduce the required number of motion primitives for planning.

Specifically, invariance/equivariance enables the adaptation of primitives on-the-fly to match some components of the state space exactly in our search algorithm, Db-A*.
This significantly reduces the number of primitives required to cover the state space, resulting in smaller discontinuities, reduced memory requirements, and faster nearest neighbor searches.

A prominent example is translation invariance, a property that holds for many mobile robots, such as differential-drives, cars, airplanes, and multirotors.
In these systems, we can decompose the state into two components, $\vx = [\vx^t, \vx^r ]$, where $\vx^t$ represents translations and $\vx^r$ contains rotations and possibly velocities.
The dynamics $f(\vx,\vu)=f(\vx^r,\vu)$ only depend on the non-translation part of the state.

We are interested in invariances that preserve optimality.
For instance, the translation of a primitive in translation invariant systems retains optimality for a running cost of type $j(\vx,\vu) = r(\vu)$, e.g., a minimum time trajectory $r(\vu) = 1$, or minimum control effort $r(\vu) = \norm{\vu}^2$.

\begin{example}[Translation Invariance in the Unicycle]
	\label{ex:unicycle:invariance}
	Consider the unicycle from \cref{ex:unicycle}.
	The dynamics \( f(\vx,\vu) = f(\theta, \vu) = [v \cos(\theta) , v \sin(\theta) , w] \) depend only on the orientation \(\theta\), but not on the position \([x,y]\).
	Using translation invariance, we can translate a motion primitive \( m=( \seqX,\seqU,\vx_s,\vx_f, c) \) with \( \mathbf{t} \in \RR^2 \), resulting in \( m \oplus \mathbf{t} = m' = ( \seqX',\seqU'=\seqU,\vx_s',\vx_f', c'=c) \).
	The states in \( \seqX' \), \( \vx_s' \), and \( \vx_f' \) are transformed as follows:
	\begin{subequations}
		\begin{align}
			\vx'[x]      & = \vx[x] + \mathbf{t}[x] \,, \\
			\vx'[y]      & = \vx[y] + \mathbf{t}[y] \,, \\
			\vx'[\theta] & = \vx[\theta]\,.
		\end{align}
	\end{subequations}
	The operator \( [\bullet] \) indicates the $\bullet$-component of a state \( \vx \) or translation vector \( \mathbf{t} \) (e.g., \( \vx[x] \in \mathbb{R} \) is the ``x-component'' of the state \( \vx \)).
\end{example}

In some second-order systems, such as the quadrotor, the acceleration depends only on the orientation and the angular velocity but is invariant to both translation and linear velocity.
Thus, we can modify primitives to match the starting position and velocity.

\begin{example}[Translation and Linear Velocity Invariance in the Quadrotor]
	\label{ex:quad:invariance}
	The second-order dynamics of the quadrotor from \cref{ex:quadrotor} depend only on the rotation and angular velocity \( \mathbf{q}, \mathbf{w} \) but not on the position \( \mathbf{p} \) or linear velocity \( \mathbf{v} \).
	We can transform a motion primitive \( m=( \seqX,\seqU,\vx_s,\vx_f, c) \) with \( \mathbf{t} = [\mathbf{t}_p, \mathbf{t}_v] \), \( \mathbf{t}_p \in \RR^{3} \), \( \mathbf{t}_v \in \RR^{3} \), resulting in \( m \oplus \mathbf{t} = m' = ( \seqX',\seqU'=\seqU,\vx_s',\vx_f', c'=c) \).
	The states in \( \seqX' \), \( \vx_s' \), and \( \vx_f' \) are transformed as follows:
	\begin{subequations}
		\begin{align}
			\vx'[v]   & = \vx[v] + \mathbf{t}_v,                              \\
			\vx'[w]   & = \vx[w],                                             \\
			\vx'[q]   & = \vx[q],                                             \\
			\vx'_k[p] & = \vx_k[p] + \mathbf{t}_p + k \mathbf{t}_v \Delta t .
		\end{align}
	\end{subequations}
\end{example}

While invariance and equivariance are advantageous properties, they require an individual study of each new dynamical system.
For simplicity, in our work, we focus only on translation invariance and linear velocity invariance as two classical and ubiquitous properties.
Additional properties, such as rotation invariance (for car-like robots), rotation symmetries (for 3D quadcopters), or angular rotational invariance (for planar multirotors), could be exploited to improve performance in particular systems.

\section{Theoretical Properties}
\label{sec:theory}

In this section, we analyze the theoretical properties of iDb-A* and argue that it is asymptotically optimal under very mild assumptions.
Intuitively, given enough computational time, iDb-A* (\cref{alg:overview}) will compute the optimal solution, because in each iteration, we add more primitives and reduce the allowed discontinuity \(\delta\), eventually producing an initial guess that is close to the optimal solution, which is then locally repaired and optimized with trajectory optimization.

\subsection{Kinodynamic Motion Planning}

\begin{assumption}
	\label{assu:delta_robust}
	The dynamics function \cref{eq:dynamics} has a bounded Lipschitz constant for both states and controls.
	Moreover, there exists a \(\delta\)-robust trajectory that solves the kinodynamic motion planning problem.
\end{assumption}
\emph{Note:}
A \(\delta\)-robust trajectory has both obstacle clearance and dynamical clearance of \(\delta\).
These assumptions are standard in kinodynamic motion planning, and we refer to \cite{ST-RRT-Star} for a formal definition.

\begin{remark}
	In a system with Lipschitz dynamics, the discrete-time dynamics converge to the continuous-time dynamics as the time step discretization approaches zero.
	Thus, we limit our study to the time-discretized system \cref{eq:dynamics_discrete} used in our framework.
\end{remark}

\subsection{Db-A*}

\begin{theorem}
	Sequences \(\seqX\) and \(\seqU\) returned by Db-A* (\cref{alg:dbAstar}) are \(\delta\)-discontinuity bounded solutions to the given motion planning problem (\cref{definition:discontinuityBounded}).
\end{theorem}
\begin{proof}
	When Db-A* terminates, we trace back the solution by following the parent pointers and obtain a sequence of motion primitives \([m_1, \ldots, m_N]\).
	This sequence defines the sequence of controls \(\seqU = [m_1.
	\seqU, \ldots, m_N.\seqU]\), and states \(\seqX = [m_1.\seqX_{[:-1]}, \ldots, m_N.\seqX]\), i.e., we take all but the last state for each primitive except for the last one, which should also include the last state.
	During Db-A*, we expand a node \(n\) with a motion \(m\) if the start state \(m.
	\vx_s\) is at most \(\alpha\delta\) away from the current state \(n.\vx\).
	Any node \(n\) is reached by a motion \(\tilde{m}\) that ends at point \(\tilde{m}.
	\vx_f\), which is at most \((1-\alpha)\delta\) from \(n.\vx\).
	Using the triangle inequality of the metric space, \(d(\tilde{m}.
	\vx_f , m.\vx_s) \leq d(\tilde{m}.\vx_f , n.\vx) + d(n.\vx, m.\vx_s) \leq \delta\).
	Thus, \cref{eq:disBound:dynamics} holds for all connections between motion primitives, and \cref{eq:disBound:xs} holds for the state in the first motion.
	We already know that \(\vx_k \in \sX\), \(\vu_k \in \sU\) for all \(k\) \cref{eq:disBound:U}.
	Motions are only used as edges if the entire motion is within \(\sX_{\mathrm{free}}\), thus \cref{eq:disBound:Xfree} holds.
	Finally, \cref{eq:disBound:xf} holds by the termination conditions.
\end{proof}

\begin{definition}
	Given a start state $\vx_s$ and a goal state $\vx_g$, a set of motion primitives $\mathcal{M}$, and a discontinuity bound $\delta$, we define the implicit graph $G_{\delta,\mathcal{M}} = (V,E)$ with,
	\begin{subequations}
		\begin{align}
			V & = \bigcup_{m\in\mathcal M} \{m.
			\vx_s,  m.\vx_f \}
			\cup \{\vx_s, \vx_g\}  \,,                                                                                      \\
			E & = \{ (m.\vx_s, m.\vx_f) ~ | ~  m \in \mathcal{M} \}\, \cup   \{ (m.\vx_f , m'.\vx_s)  \nonumber             \\                                                                                                            & ~ | ~  m,m' \in \mathcal{M} ~ \textbf{and} ~ d(m.\vx_f , m'.\vx_s) \leq \delta \}
			\nonumber                                                                                                       \\                                                                                                          &
			\cup \{ (\vx_s, m.\vx_s) ~ |  ~  m \in \mathcal{M}  ~\textbf{and} ~ d(\vx_s , m.\vx_s) \leq \delta \} \nonumber \\                                                                                                          &
			                                                                                                           \cup \{ ( m.\vx_f, \vx_g) ~|  ~  m \in \mathcal{M} ~\textbf{and} ~ d(m.\vx_f , \vx_g) \leq \delta \}.
		\end{align}
	\end{subequations}
\end{definition}

\begin{remark}
	\label{th:incomplete}
	For an arbitrary fixed value of \( \delta > 0 \), Db-A* is incomplete and suboptimal when searching on the implicit graph \( G_{\mathcal{M},\delta} \).
\end{remark}
\begin{proof}
	Consider an example where a robot has to move through a narrow door to navigate to an adjacent room.
	Even if a \( \delta \)-discontinuity bounded solution using the available primitives exists (i.e., a path in the graph $G_{\delta,\mathcal{M}}$), Db-A* may not find it because motions are expanded in a random order and only if no previous node is within distance \( (1-\alpha)\delta \).
	This could potentially prune a node that is required in the solution.
	Since Db-A* is incomplete, it cannot guarantee that no better \( \delta \)-discontinuity bounded solution exists once it finds one.
\end{proof}

\begin{theorem}
	\label{theo:dbastar-complete}
	For a given set of motion primitives \( \mathcal{M} \), there exists a \( \delta >0 \) such that Db-A* is complete and optimal when searching in the implicit graph \( G_{\mathcal{M}, \delta} \).
\end{theorem}
\begin{proof}
	If the choice of a given \( \delta_0 \) renders Db-A* incomplete on the graph \( G_{\mathcal{M}, \delta_0} \), we can always reduce \( \delta \), resulting in a different graph, where the nodes in the solution are either not pruned during the Db-A* search, or the graph becomes unsolvable.
	Note that since decreasing \( \delta \) may render the graph unsolvable, this property is valuable in an asymptotic setting, used to demonstrate asymptotic optimality when incrementally adding more motion primitives.
\end{proof}

\subsection{Optimization}
\begin{assumption}(Convergence of the Optimizer)
	\label{assu:convergence_opti}
	Let \( (\seqX_{d},\seqU_{d}) \) be a \( \delta \)-discontinuity bounded solution (\cref{definition:discontinuityBounded})
	to the planning problem.
	Then, there exists a (small) \( \delta > 0 \) such that trajectory optimization converges to a locally optimal solution $(\seqX^{*} , \seqU^{*})$ of the original kinodynamic motion planning problem \cref{eq:motion-planning}.
\end{assumption}

\emph{Discussion:}
Nonlinear constrained optimization algorithms for trajectory optimization have convergence guarantees toward stationary points: these are either points that satisfy the first-order optimality conditions—thereby converging to a locally optimal feasible solution; or points that locally minimize the constraint violation.

In general, these methods typically require the dynamics, distance function, and collision constraints to have smoothness or bounded derivatives.
The exact conditions for convergence can vary slightly among different trajectory optimization algorithms, such as Differential Dynamic Programming \cite{murray1984differential}, Sequential Convex Optimization \cite{bonalli2019gusto}, the Augmented Lagrangian \cite{fernandez2012local}, and Interior Points \cite{wachter2005line}.

Specifically, we assume that for a small \( \delta \), the initial guess is close to a feasible solution, thereby ensuring that the optimizer converges to a feasible and locally optimal solution.
We argue that this is a very reasonable assumption, especially when the optimal solution is
a \( \delta \)-robust trajectory (\cref{assu:delta_robust}), as often assumed in the literature.
The experimental success rate of our planner (\cref{sec:experiments}) also supports this assumption.

\subsection{Motion Primitives}

Our optimization-based approach for generating motion primitives, as well as propagation of random control inputs from random starting points, creates a set of motion primitives that asymptotically covers the state space, which will be required to prove the asymptotic optimality of iDb-A*.

\begin{definition}
	\label{def:cover}
	A set of primitives $\mathcal{M}$ covers the state space with discontinuity $\epsilon >0$ if and only if, for all pairs of states $\vx,\vx' \in \mathcal{X}$, there exists a $\epsilon$-discontinuity bounded trajectory from $\vx$ to $\vx'$ (\cref{definition:discontinuityBounded}) using the primitives $m\in \mathcal{M}$.
\end{definition}

\begin{definition}
	\label{def:asym-cover}
	A method to generate motion primitives asymptotically covers the state if, for every $\tilde{\epsilon} > 0$, there exists a \emph{finite} set of primitives $ \tilde{\mathcal{M}} $, $ | \tilde{\mathcal{M}} | < \infty $ such that \cref{def:cover} holds.
	In other words, asymptotic coverage means that if $|\mathcal{M}| \to \infty $, then $ \epsilon \to 0 $ in \cref{def:cover}.
\end{definition}

\begin{remark}
	\label{th:primitives}
	Motion primitives generated with our randomized optimization-based approach, as well as rollouts of random control inputs from random starting points, asymptotically covers the state space.
\end{remark}

\subsection{iDb-A*}

So far, we have demonstrated that there exists a small discontinuity bound \(\delta > 0\) such that Db-A* finds an optimal discontinuity bounded trajectory (if one exists), and that trajectory optimization can be used to repair a \(\delta\)-discontinuity bounded solution into a locally optimal solution.

To prove asymptotic optimality for iDb-A*, we use techniques from sampling-based motion planning to show that, as we increase the number of primitives, the solution of Db-A* converges towards an optimal discontinuity bounded solution (and thus, after optimization, we converge towards the optimal solution).
In \cref{alg:overview}, we iteratively reduce \(\delta\) to achieve good anytime behavior.
For the proof, it is sufficient to assume that we are using a small fixed discontinuity \(\tilde{\delta} > 0\) such that \cref{theo:dbastar-complete} and \cref{assu:convergence_opti} hold.

\begin{theorem}
	\label{theorem:ao}
	If the set of motion primitives asymptotically optimally covers the state space with discontinuity \(\tilde{\delta}\), and assuming convergence of the optimizer, iDb-A* (\cref{alg:overview}) is asymptotically optimal, i.e., \begin{equation} \lim_{n\to\infty} P(\{ c_n - c^* > \epsilon \}) = 0, \; \forall \epsilon > 0, \end{equation} where \(c_n\) is the cost in iteration \(n\) and \(c^*\) is the optimal cost.
\end{theorem}

\begin{proof}
	We closely follow the proof strategy from previous work in sampling-based motion planning \cite[Th.
		3]{AO-RRT}.
	Let \(S_1,\ldots,S_n\) be random variables denoting the suboptimality \(c_n - c^*\).
	In every iteration of \cref{alg:overview}, we either reduce the cost if we find a new solution or we remain at the same cost, i.e., \(c_{n+1} \leq c_n\).
	We now aim to show that, with sufficient motion primitives, the solution of Db-A* will be close to the true solution, and thus can be correctly optimized by trajectory optimization.
	Crucially, in each iteration of iDb-A*, there is a positive probability that the new primitives will improve the solution, \(E[S_n|S_{n-1}] \leq (1-\omega)S_{n-1}\), i.e., in expectation, the solution improves by at least a constant amount \(\omega > 0\) every iteration.
	This nonzero probability only holds if the motion primitives asymptotically cover the entire state space with discontinuity \(\tilde{\delta}\) (\cref{def:asym-cover}).
	Then, we have
	\begin{align}
		E[S_n] & = \int E[S_n | S_{n-1}] P(S_{n-1}) dS_{n-1}               \\
		       & \leq (1-\omega) \int S_{n-1} P(S_{n-1}) dS_{n-1}\nonumber \\
		       & = (1-\omega)E[S_{n-1}] = (1-\omega)^{n-1} E[S_1].
		\nonumber
	\end{align}
	Applying the Markov inequality, we have \(P(S_n > \epsilon) \leq E[S_n]/\epsilon = (1-\omega)^{n-1} E[S_1] / \epsilon\), which approaches 0 as \(n\) approaches infinity.
\end{proof}

It remains an open question what the theoretical convergence rate of our proposed algorithm is, a property that is known for some sampling-based planners \cite{AO-RRT-Analysis}.
Empirically, we have shown that our initial solution is often much closer to the optimum compared to our baselines, and that the region of attraction for trajectory optimization is large enough to plan with a few primitives and large discontinuity bounds.

\section{Experimental Evaluation}

\label{sec:experiments}

\begin{figure*}[ht]
	\centering
	\begin{subfigure}[b]{.17\textwidth}
		\centering
		\includegraphics[width=.95\linewidth]{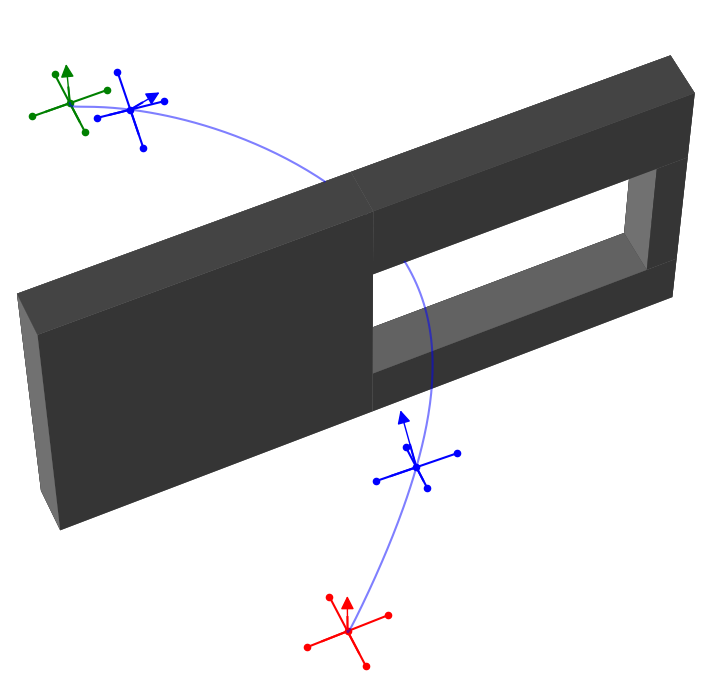}
		\caption{\label{fig:dyno:quad3d}}
	\end{subfigure}%
	\begin{subfigure}[b]{.17\textwidth}
		\centering
		\includegraphics[width=.95\linewidth]{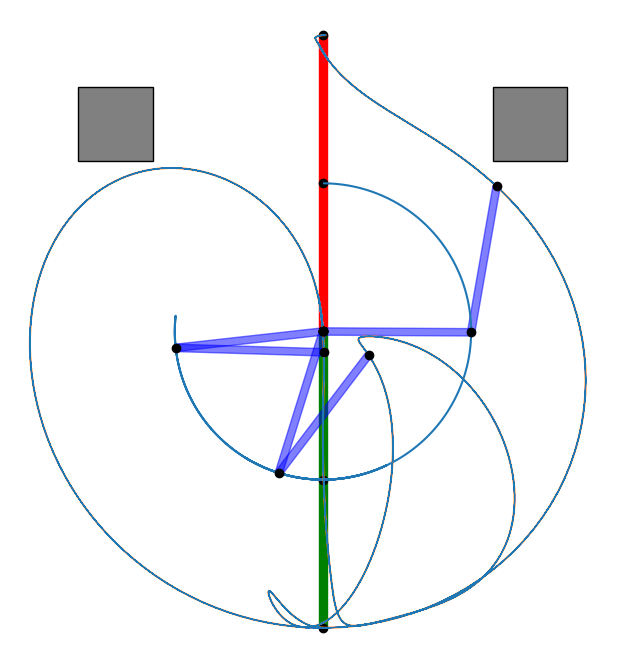}
		\caption{\label{fig:dyno:acrobot}}
	\end{subfigure}%
	\begin{subfigure}[b]{.17\textwidth}
		\centering
		\includegraphics[width=.95\linewidth]{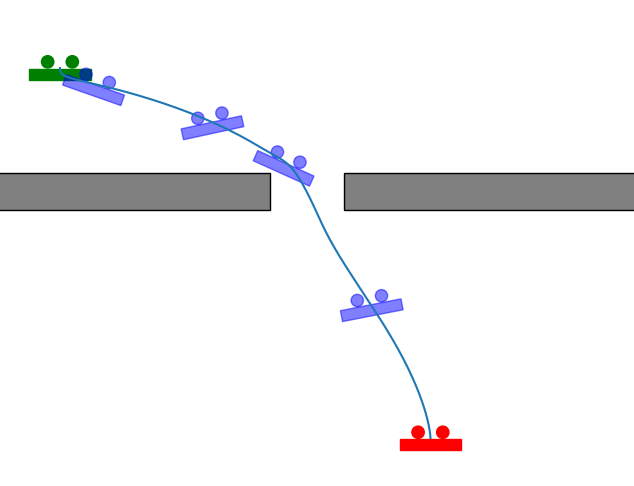}
		\caption{\label{fig:dyno:quad2d}}
	\end{subfigure}%
	\begin{subfigure}[b]{.17\textwidth}
		\centering
		\raisebox{5mm}{
			\includegraphics[width=.99\linewidth]{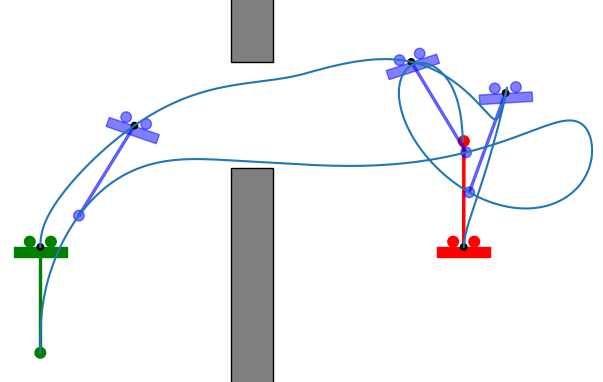}}
		\caption{\label{fig:dyno:quad2dpole}}
	\end{subfigure}%
	\begin{subfigure}[b]{.17\textwidth}
		\centering
		\raisebox{5mm}{
			\includegraphics[width=.95\linewidth]{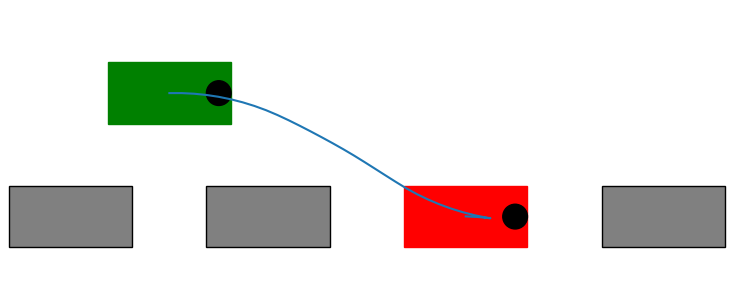}}
		\caption{\label{fig:dyno:uni1}}
	\end{subfigure}%
	\caption{Kinodynamic motion planning problems:
		Start and goal positions are represented in green and red, respectively.
		Obstacles are depicted in gray.
		Trajectories computed by iDb-A* are illustrated in blue.
		(a) \textit{Quadrotor v1 -- Window}, (b) \emph{Acrobot -- Swing up obstacles v1}, (c) \emph{Planar rotor -- Hole}, (d) \emph{Rotor pole -- Small window}, (e) \emph{Unicycle 1 v0 -- Park}.}
	\label{fig:overview}
\end{figure*}

We evaluate iDb-A* on 43 problems that include 8 different dynamical systems in various environments.
Most of the problems and systems are selected from previous work in kinodynamic motion planning \cite{SSTstar, shomeAsymptoticallyOptimalKinodynamic2021b, hoenig2022benchmarking, granados2022towards}.
Additionally, we include several problems that require dynamic and agile maneuvers with multirotors.

The benchmark problems are available in \textit{DynoBench}, our new benchmark library.
It provides a C++ implementation of all the dynamical systems (including dynamics with analytical Jacobians, state, and bound constraints), collision and distance computation with the Flexible Collision Library (FCL), the environments (in human-friendly YAML files), and visualization tools in Python.

Our implementation of iDb-A* and the other planners is available in our repository, along with the motion primitives and instructions to replicate the benchmark results.
A visualization of each problem and the corresponding solution trajectories computed by our algorithm is available on our website.

\subsection{Dynamical Systems and Environments}

We include a diverse range of dynamical systems and environments, featuring varying state dimensionality (from 3 to 14), number of underactuated degrees of freedom, and controllability.

\begin{enumerate}
	\item \textbf{Unicycle 1 (\(1^{\text{st}}\) order)} has a 3-dimensional state space \([x,y,\theta] \in \sX \subset \mathbb{R}^2 \times SO(2)\) and a 2-dimensional velocity control \([v, \omega] \in \sU \subset \mathbb R^2\) \cite{lavallePlanningAlgorithms2006}.
	      The three variants (\(\textup{v}0\), \(\textup{v}1\), \(\textup{v}2\)) use different control bounds.
	      See \cref{fig:overview:uni1,fig:dyno:uni1}.

	\item \textbf{Unicycle 2 (\(2^{\text{nd}}\) order)} has a 5-dimensional state space \([x,y,\theta,v,\omega] \in \sX \subset \mathbb R^4 \times SO(2)\) and a 2-dimensional acceleration control \([\dot{v}, \dot{\omega}] \in \sU \subset \mathbb R^2\) \cite{lavallePlanningAlgorithms2006}.
	      See \cref{fig:unicycle2_bugtrap}.

	\item \textbf{Car with trailer} has a 4-dimensional state space \([x,y,\theta_0, \theta_1] \in \sX \subset \mathbb R^2 \times SO(2)^2\) and a 2-dimensional control space \([v, \phi] \in \sU \subset \mathbb R^2\) (steering angle and velocity) \cite{lavallePlanningAlgorithms2006}.
	      See \cref{fig:overview:car}.

	\item \textbf{Acrobot} is a two-link planar manipulator actuated only at the middle joint \cite{underactuated}.
	      It requires long trajectories to swing up the two links.
	      See \cref{fig:dyno:acrobot}.

	\item \textbf{Quadrotor v0} has a 13-dimensional state space (position, orientation, and first-order derivatives) and a 4-dimensional control space (force for each of the four motors).
	      Dynamics are defined in \cref{ex:quadrotor}, and we use the parameters of the Crazyflie 2.1.
	      The low thrust-to-weight ratio of \num{1.3} is very challenging for kinodynamic motion planning, and harsh initial conditions prevent the use of specialized methods \cite{liuSearchBasedMotionPlanning2018,zhouRobustEfficientQuadrotor2019}.
	      See \cref{fig:overview:3d_quad}.

	\item \textbf{Quadrotor v1}.
	      The state space is the same as in \textit{Quadrotor v0}.
	      Controls are the total thrust and torques in the body frame, which make sampling-based methods more efficient but approximate real rotor-force limits.
	      We use the system parameters and model from the OMPL APP repository, but increase the control bounds to perform agile maneuvers.
	      See \cref{fig:dyno:quad3d}.

	\item \textbf{Planar rotor}.
	      The input is the force in each rotor \(\vu = [ f_1, f_2 ] \in \RR^2\) and the state space is 6-dimensional \( \vx = [ x, z, \theta , v_x, v_y , w ] \in \mathbb{R}^5 \times SO(2) \).
	      The thrust-to-weight ratio is also limited to \num{1.3}.
	      See \cref{fig:dyno:quad2d}.

	\item \textbf{Rotor pole} (Planar multirotor with pole) is a planar multirotor with an additional underactuated pendulum.
	      The control space is the same as in \emph{Planar rotor} (but with larger control bounds) and the state space has two additional degrees of freedom \([q,\dot{q}]\) for the pendulum.
	      See \cref{fig:dyno:quad2dpole}.
\end{enumerate}

All systems use the explicit Euler integration \eqref{eq:dynamics_discrete}, with \(\Delta t=\SI{0.1}{s}\) for all car-like robots, and \(\Delta t=\SI{0.01}{s}\) for the flying robots and the \textit{Acrobot}, due to the fast rotational dynamics.

For most car-like robots, we consider three environments (\textit{Kink}, \textit{Park}, \textit{Bugtrap}).
For the \textit{Acrobot} and \textit{multirotors}, we use environments with and without obstacles to evaluate performance in settings that require both aggressive maneuvers (e.g., recovering from upside-down positions) and navigation around obstacles.

\subsection{Algorithms}

Following our previous work \cite{hoenig2022dbAstar}, the goal of this benchmark is to compare methods for kinodynamic motion planning that use different methodologies: search, optimization, and sampling.
We compare our algorithm against state-of-the-art methods that have available open-source implementations.

For a \textbf{search-based} approach, we rely on SBPL (Search-based Planning Library), a commonly used C++ library.
We generate our own primitives and make minor adjustments to the heuristic to enable time-optimal anytime planning using the provided implementation of ARA* in SBPL.
SBPL requires the motion primitives to be connected without discontinuity and to span a lattice.
We limit our evaluation to dynamic models that are readily available in SBPL, namely the \textit{Unicycle 1 v0}.

For a \textbf{sampling-based} approach, we use SST*~\cite{SSTstar}, which is implemented in OMPL~\cite{OMPL} (Open Motion Planning Library).
Since sampling-based kinodynamic approaches cannot reach a goal state exactly, we use a goal region instead and run subsequent trajectory optimization with fixed terminal time to generate an exact solution to the goal.
The reported time does \emph{not} include the time spent in trajectory optimization, thus providing a favorable lower bound.

For \textbf{optimization-based} planning, we choose a classic combination of a geometric motion planner and a trajectory optimizer, which we call \ALGrrt in the following.
The motion planner generates an obstacle-free initial guess, ignoring the dynamics of the system, i.e., \emph{planning with a simplified model}, and the optimizer uses this trajectory as an initial guess.
As a motion planner, we use geometric RRT* (using the implementation in OMPL), which provides anytime behavior by incrementally improving the geometric trajectory.
Importantly, the probability of finding a feasible solution with optimization might increase when using the multiple geometric initial guesses provided by RRT*.
The initial guess provided by the geometric planner is collision-free but is often not accurate for dynamic constraints.
Therefore, for trajectory optimization with free terminal time, we use the strategy \textit{Search-T} \cref{eq:search-t}, which is more robust and provides better success in this planner than \textit{Free-dt}, later used in \ALGidbas.

For \textbf{iDb-A*}, we implement \cref{alg:overview} and \cref{alg:dbAstar} in C++.
As a heuristic \( h \), we use the Euclidean-based heuristic (e.g., distance divided by the upper bound of the speed, see \cref{sec:heuristic}), because it is general, fast to compute, and does not require precomputation.
For trajectory optimization, we use \textit{Free-dt} \eqref{eq:dt}, which provides a good balance between local optimality and computational time when using a good initial guess, as usually provided by Db-A*.
We provide an ablation study of the choices of heuristic and trajectory optimization strategy in \cref{sec:ablation}, where we analyze and compare different alternatives.

The rates that control the number of primitives and the discontinuity bound in each iteration of \ALGidbas (\cref{alg:overview}) are set to \( n_r = 1.5 \), \( \delta_r = 0.9 \) in all problems.
However, when the search terminates without finding a solution, we keep the discontinuity bound (almost) constant with \( \delta_r = 0.999 \) and only increment the number of primitives with a rate of \( n_r = 1.5 \).
The parameter \( \alpha \) in \cref{alg:dbAstar} is set to 0.5.
The initial number of motion primitives, discontinuity bound, and hyperparameters of the trajectory optimization are chosen per dynamical system.
For example, in all problem instances with \qsystem{Unicycle 1 v0}, we start with 100 primitives and a discontinuity \( \delta_0 = 0.3 \), and in all instances with \qsystem{Quadrotor v0}, we start with 2000 primitives and \( \delta_0 = 0.7 \).
The distance function \( d \) used to measure the distance between states in \ALGdb is a weighted Euclidean norm, which uses different weights for the position, orientation, and velocity components of the state (e.g., in the \textit{Unicycle 2} we use weights 1, 0.5, and 0.25, respectively).

The offline generation of a valid primitive takes a few seconds, from \SI{0.1}{s} for the \textit{Unicycle 1 v0} to \SI{6}{s} for a \textit{Quadrotor v0} using a single core of a laptop computer with a CPU {\small i7-1165G7@2.80GHz}, where most of the time is spent attempting to solve random two-value boundary problems, which cannot be solved with short trajectories.

The benchmarking infrastructure is written in Python, and all tuning parameters for our algorithms and baselines can be found in our open-source repository.

For nearest-neighbors computation in \ALGidbas, \ALGsst, and \ALGrrt, we use NIGH \cite{ichnowski2015fast} in single-core mode, which provides faster lookup and insertion times than the default Geometric Near-neighbor Access Tree (GNAT) implementation in OMPL.
All trajectory optimization algorithms are implemented based on the DDP solver in Crocoddyl \cite{Crocoddyl}.

\subsection{Metrics}

In the following, we report the metrics:
\begin{enumerate}
	\item Success Rate (\(p\)): The ratio of trials where a solution was found within the planning budget of \SI{120}{s}.
	\item Median time required to find a solution (\(t^{\text{st}}\)).
	\item Median cost when 50\% of the trials have found a solution (\(J^{\text{st}}\)).
	\item Median cost of the final solution (\(J^{\text{f}}\)) found within the planning budget of
	      \SI{120}{s}.
\end{enumerate}

\subsection{Benchmark}

\begin{figure}
	\setlength{\tabcolsep}{0.0em}
	\centering
	\begin{tabular}{ccc}
		\includegraphics[width=0.32\linewidth]{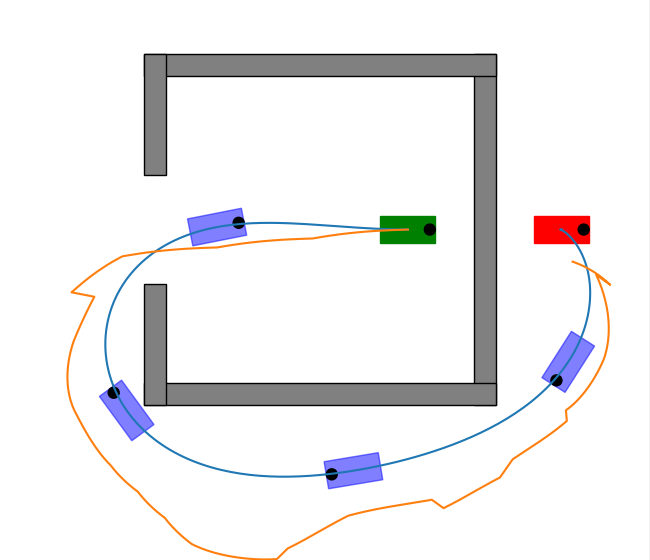} &
		\includegraphics[width=0.32\linewidth]{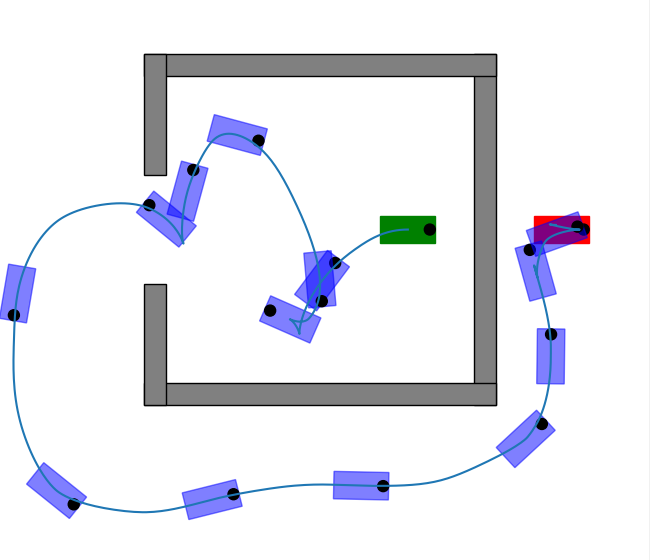}      &
		\includegraphics[width=0.32\linewidth]{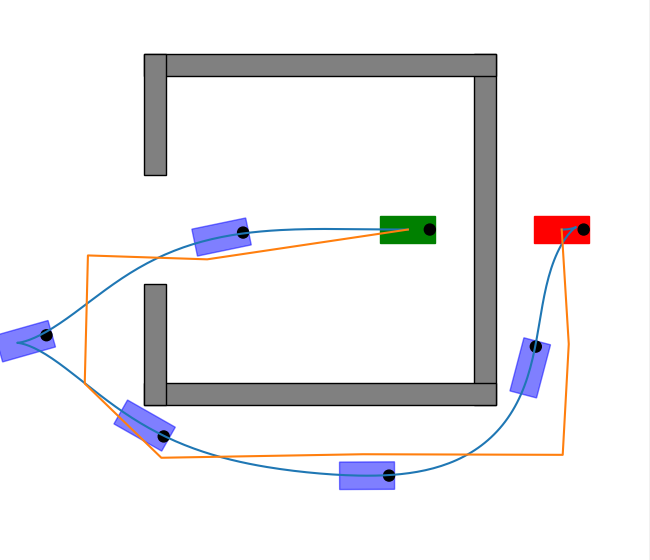}                                                                                  \\
		\small{\ALGidbas (\(J\)=\SI{25.9}{s})}                                        & \small{\ALGsst (\(J\)=\SI{75}{s})} & \small{\ALGrrt (\(J\)=\SI{30.9}{s})}
	\end{tabular}
	\caption{Example of the first trajectory computed by \ALGidbas, \ALGsst, and \ALGrrt in \emph{Unicycle 2 -- Bugtrap}.
		The orange line in \ALGidbas and \ALGrrt shows the initial guess trajectory before optimization.
		The first solution of \ALGidbas has the lowest cost (\(J\)).
		The average computational time to get the first solution is \SI{1.2}{s} in \ALGidbas, \SI{2.8}{s} in \ALGrrt, and \SI{0.9}{s} in \ALGsst.
	}
	\label{fig:unicycle2_bugtrap}
\end{figure}

\label{sec:benchmark}

We conducted our benchmark on a workstation with a CPU {\small AMD EPYC 7502 32-Core Processor @2.50 GHz}.
All planners use a single core.
Our results are summarized in \cref{tab:the-table-selected}, where we provide a selection of 16 problems (two for each dynamical system).
For brevity, \cref{tab:the-table-selected} does not include any standard deviations.
The complete results (43 problems) are available on the project webpage.
In general, we found that \texttt{SBPL} has almost no variance, \ALGsst has very high variance, and \ALGrrt and \ALGidbas are somewhere in between the two extremes.
The plots in \cref{fig:interesting} show the convergence behavior and the variance in three representative problems, which are discussed later.

We summarize the main results as follows:

$\bullet$ \texttt{SBPL} has been excluded from the table because it is only readily applicable to the three problems that use \emph{Unicycle 1 v0}.
In this setting, it consistently finds a solution in competitive time: \SI{2.1}{s} in \emph{Bugtrap}, \SI{0.2}{s} in \emph{Kink}, and \SI{0.1}{s} in \emph{Park}.

However, due to the limited number of lattice-based primitives, the initial and final costs are rather constant: \num{36.6} in \emph{Bugtrap}, \num{22.6} in \emph{Kink}, and \num{6.2} in \emph{Park} (and higher than the costs achieved by \ALGidbas).

\begin{table*}[ht]
	\caption{Benchmark with selected problems.
		\textbf{Bold} indicates the best result.}
	\centering
	\begin{tabular}{lllrrrrrrrrrrrr}
\toprule
& & &  \multicolumn{4}{c}{iDb-A*} & \multicolumn{4}{c}{SST*} & \multicolumn{4}{c}{RRT*+TO}\\
\cmidrule(lr){4-7}\cmidrule(lr){8-11}\cmidrule(lr){12-15}
 \# & System & Instance & $p$ & $t^{\textup{st}} [s]$ & $J^{\textup{st}}[s]$ & $J^{\textup{f}}[s]$ & $p$ & $t^{\textup{st}} [s]$ & $J^{\textup{st}}[s]$ & $J^{\textup{f}}[s]$ & $p$ & $t^{\textup{st}} [s]$ & $J^{\textup{st}}[s]$ & $J^{\textup{f}}[s]$\\
\midrule
0 & Acrobot & Swing up & \textbf{1.0} & \textbf{1.3} & \textbf{5.2} & 4.9 & 0.6 & 16.4 & 5.3 & 4.6 & 0.6 & 2.0 & 8.1 & \textbf{4.0}\\
1 & Acrobot & Swing up obstacles v1 & \textbf{1.0} & \textbf{1.7} & 6.4 & 5.1 & 0.4 & - & - & - & 0.9 & \textbf{1.7} & \textbf{6.1} & \textbf{3.9}\\
2 & Car with trailer & Kink & \textbf{1.0} & 1.6 & \textbf{26.9} & 24.4 & \textbf{1.0} & \textbf{0.6} & 75.8 & 63.1 & \textbf{1.0} & 2.0 & 44.8 & \textbf{18.1}\\
3 & Car with trailer & Park & \textbf{1.0} & 0.4 & 19.7 & \textbf{4.6} & 0.7 & 1.6 & 18.4 & 9.9 & 0.8 & \textbf{0.2} & \textbf{5.7} & 5.3\\
4 & Planar rotor & Hole & 0.7 & 34.2 & \textbf{3.8} & 3.8 & 0.9 & 39.8 & 13.5 & 8.0 & \textbf{1.0} & \textbf{1.1} & 8.1 & \textbf{3.4}\\
5 & Planar rotor & Bugtrap & \textbf{1.0} & 12.3 & \textbf{5.5} & \textbf{5.2} & \textbf{1.0} & 26.4 & 13.6 & 8.6 & \textbf{1.0} & \textbf{1.9} & 11.6 & 8.1\\
6 & Rotor pole & Swing up obstacles & \textbf{1.0} & \textbf{2.1} & \textbf{4.0} & \textbf{3.9} & 0.0 & - & - & - & 0.8 & 19.8 & 5.6 & 4.1\\
7 & Rotor pole & Small window & \textbf{1.0} & \textbf{3.2} & \textbf{4.5} & \textbf{4.5} & 0.0 & - & - & - & 0.5 & - & - & -\\
8 & Quadrotor v0 & Recovery & \textbf{1.0} & \textbf{0.9} & \textbf{6.2} & \textbf{2.6} & 0.0 & - & - & - & 0.0 & - & - & -\\
9 & Quadrotor v0 & Recovery obstacles & \textbf{1.0} & \textbf{1.5} & \textbf{5.3} & \textbf{3.6} & 0.0 & - & - & - & 0.9 & 8.8 & 6.9 & 3.6\\
10 & Quadrotor v1 & Obstacle & \textbf{1.0} & \textbf{1.2} & \textbf{2.7} & \textbf{2.4} & 0.2 & - & - & - & \textbf{1.0} & 1.5 & 4.4 & 3.5\\
11 & Quadrotor v1 & Window & \textbf{1.0} & \textbf{1.7} & \textbf{2.6} & \textbf{2.2} & 0.3 & - & - & - & 0.8 & 21.1 & 7.4 & 3.7\\
12 & Unicycle 1 v0 & Bugtrap & \textbf{1.0} & 0.5 & \textbf{22.3} & \textbf{21.0} & \textbf{1.0} & \textbf{0.1} & 70.5 & 23.8 & \textbf{1.0} & 1.2 & 45.1 & 23.8\\
13 & Unicycle 1 v2 & Wall & \textbf{1.0} & 0.8 & \textbf{20.8} & \textbf{18.4} & \textbf{1.0} & \textbf{0.1} & 49.7 & 19.2 & 0.1 & - & - & -\\
14 & Unicycle 2 & Bugtrap & \textbf{1.0} & 1.2 & \textbf{25.2} & \textbf{25.0} & \textbf{1.0} & \textbf{0.9} & 98.2 & 49.4 & \textbf{1.0} & 2.8 & 44.5 & 28.9\\
15 & Unicycle 2 & Park & \textbf{1.0} & \textbf{0.1} & \textbf{5.8} & \textbf{5.8} & \textbf{1.0} & 0.2 & 29.2 & 5.9 & 0.9 & \textbf{0.1} & 6.0 & \textbf{5.8}\\
\bottomrule
\end{tabular}

	\label{tab:the-table-selected}
\end{table*}

\begin{figure*}[ht]
	\centering

	\def \plott {tro_submission_x/plot_2023-11-06--20-39-29.pdf}
	\includegraphics[scale=.46, page=1]{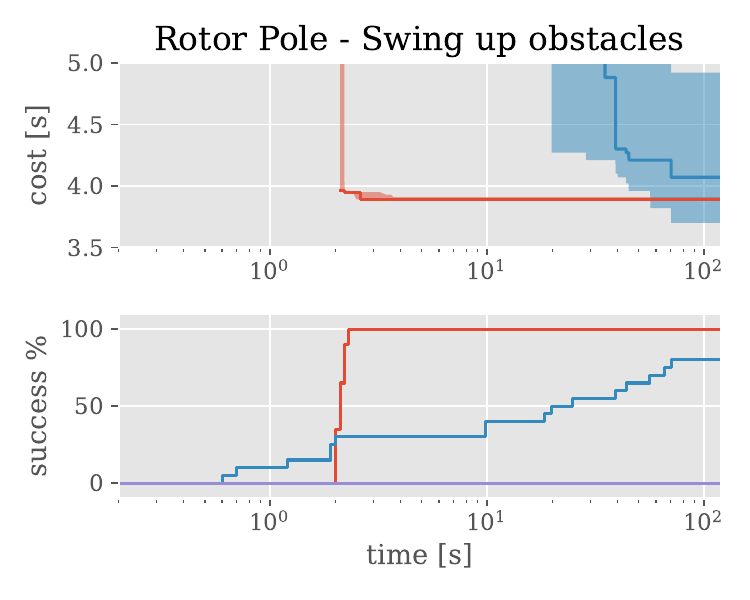}
	\includegraphics[scale=.46, page=2]{\plott}
	\includegraphics[scale=.46, page=3]{\plott}

	\caption{Success and cost convergence plots for three representative systems.
		Cost is only plotted when 50\% of the runs have found a solution.
		The shaded region indicates the 95\% non-parametric confidence interval for the median.
	}
	\label{fig:interesting}
\end{figure*}

$\bullet$ \ALGsst can find an initial solution quickly in problems with car-like dynamics; however, the quality of the initial solution is poor, especially in the larger, higher-dimensional systems.
The convergence is slow—our \SI{120}{s} timeout was not sufficient for \ALGsst to fully converge in most cases.
Notably, in problems involving multirotors, it was unable to find solutions within the time limit for most problems (for instance, we observed a 0\% success rate in the dynamics of \emph{Quadrotor v0} and 10\% with \emph{Quadrotor v1}).
Because \ALGsst relies on propagating random control inputs, it is very inefficient for multirotor systems, where random inputs quickly bring the system into unstable configurations.
However, in low-dimensional and stable dynamical systems like cars and unicycles, it finds the first solution faster but is clearly outperformed in terms of the cost of initial and final solutions by both \ALGidbas and \ALGrrt.

$\bullet$ \ALGrrt can find near-optimal initial solutions in some problems but fails if the geometric initial guess is not close to dynamically feasible motion, which occurs more often in environments with flying robots.
Thus, the main drawback is that this approach is incomplete, with success rates below 70\% on several problems and complete failures in others.
The cost at convergence is often worse than that of \ALGidbas.
In general, we conclude that \ALGrrt is a good method when the simplified model is informative and when the primary challenge is obstacle avoidance.
In these settings, it often matches the time to first solution of \ALGidbas.

$\bullet$ \ALGidbas finds the highest-quality first solution in 14 out of 16 selected problems (better in the \(J^{\text{st}}\) column), converged to the lowest-cost solution in 12 out of 16 problems (column \(J^{\text{f}}\)), and achieved a 100\% success rate in 15 out of 16 problems (column \(p\)) (and 41 out of 43 total problems).
The time to generate the first solution (column \(t^{\text{st}}\)) is competitive with the other approaches in the problems solved by all methods, while it is the only method that consistently solves all problems with multirotor flying robots.

We can conclude that our method performs well across all systems and environments, from navigation among obstacles with car models to recovery flights with control-limited quadrotors.
Note that the performance we report here is considerably better than our previous results~\cite{hoenig2022dbAstar}.
This improvement is due to an improved implementation of the search algorithm, a superior trajectory optimization algorithm and formulation, and a better strategy for choosing the number of primitives and the discontinuity bound.

\cref{fig:unicycle2_bugtrap} shows the different first solutions found by \ALGidbas, \ALGrrt, and \ALGsst in the \qsystem{Unicycle 2 -- Bugtrap} problem.
We also display the convergence and success plots for some instructive problems in \cref{fig:interesting}:

\begin{enumerate}

	\item \emph{Rotor pole -- Swing up obstacles}: This problem involves swinging the pole upwards while avoiding obstacles (\cref{fig:overview:pole}).
	      Only \ALGidbas consistently solves the problem within a competitive timeframe.
	      On average, \ALGrrt requires 10x more computational time to find a solution and does not achieve a 100\% success rate.
	      The disparate performance across runs of \ALGrrt stems from the uninformative geometric guess, which often leads to failure in the subsequent optimization and thus requires multiple trials with different initial guesses.
	      \ALGsst does not find any solution within the computational budget.

	\item \emph{Quadrotor v1 -- Window}: In this problem, the quadrotor needs to find a path through a window; see \cref{fig:dyno:quad3d}. \ALGsst achieves a low success rate because propagating random controls is often inefficient, with a low probability of generating useful trajectories. \ALGidbas consistently solves the problem in at most two seconds and improves the solution with more compute time. \ALGrrt fails to find a solution in some runs and the median cost is considerably higher.

	\item \emph{Unicycle 1 v0 -- Bugtrap}: The \emph{Bugtrap} environment is shown in \cref{fig:unicycle2_bugtrap}, but here we use the \emph{Unicycle 1 v0} system instead of \emph{Unicycle 2}.
	      \ALGsst finds the first solution the fastest, but convergence to the optimum is slow.
	      \ALGidbas is also quick and produces a high-quality solution already in the first iteration. \ALGrrt finds solutions as well but often requires multiple optimization trials to obtain the first valid solution (similar to the previous problems, \ALGrrt has a high variance in the time required to solve the problem).
	      Regarding \texttt{SBPL}, the time to reach the first solution is competitive with the other methods, but due to the limited number of primitives, the solution does not improve.
\end{enumerate}

\subsection{Ablation Studies of iDb-A*}
\label{sec:ablation}

We analyze the main algorithmic components of \ALGidbas to study the impact on the overall performance and to justify the most important design decisions, namely:

\begin{itemize}
	\item[--] Scheduling for increasing/decreasing the number of primitives and the discontinuity bound.
	\item[--] Euclidean heuristic in the search step.
	\item[--] Optimization with free terminal time using \textit{Free-dt}.
	\item[--] Optimization-based motion primitives.
	\item[--] Invariance/equivariance of motion primitives.
\end{itemize}

We find a significant interaction between hyperparameters and design decisions, making it challenging to evaluate the impact of each component.
For instance, the runtime of the search step with different heuristics is strongly dependent on the number of primitives.
Therefore, instead of including more variations of \ALGidbas in our benchmark, we choose to analyze and discuss each component individually, which we believe provides a clearer understanding.
The experiments in the ablation study are conducted on a modern laptop computer with a CPU {\small i7-1165G7@2.80GHz} (single core), which has single-core performance similar to that of the workstation used.

\subsubsection{Time Spent in Each Component}

We first evaluate how much computational time is spent in the search or optimization components, and how this varies when we decrease the discontinuity bound and increase the number of primitives.

\begin{figure}
	\centering
	\includegraphics[width=.49\textwidth]{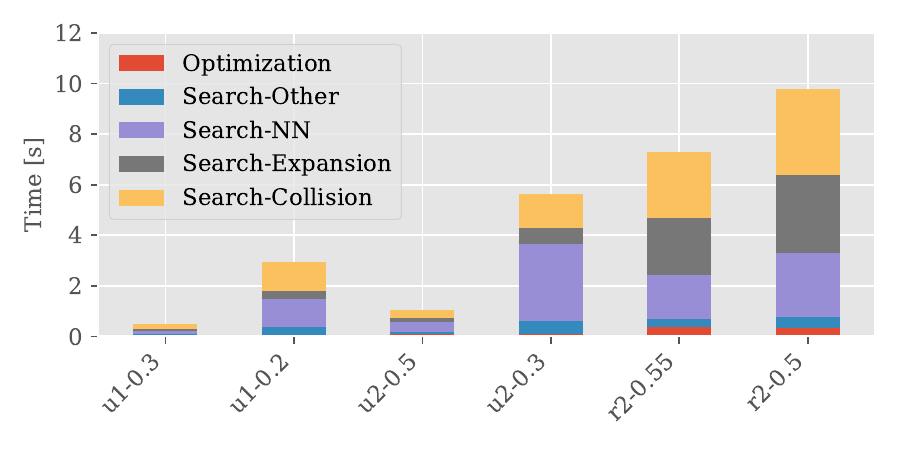}
	\caption{Time spent in the search and optimization steps during the first iteration of \ALGidbas in the \textit{Bugtrap} environment, with three different dynamical systems and two different discontinuity bounds.
		The labels \textit{u1}, \textit{u2}, and \textit{r2} are short names for the systems \textit{Unicycle 1 v0}, \textit{Unicycle 2}, and \textit{Planar rotor}.
		The number after the hyphen indicates the discontinuity bound; for example, \emph{u1-0.3} is \textit{Unicycle 1 v0} with $\delta=0.3$.
	}
	\label{fig:time-spent-each-component}
\end{figure}

\cref{fig:time-spent-each-component} shows an analysis of the computation time spent in the search and optimization during the first iteration of \ALGidbas for three different systems (\emph{Unicycle 1 v0}, \emph{Unicycle 2}, and \textit{Planar rotor}), using two different values of the discontinuity bound (with a consistent number of motion primitives), in the \textit{Bugtrap} environment (e.g., \cref{fig:unicycle2_bugtrap}).

The total time is predominantly occupied by the search component, where both collision checks and nearest-neighbor queries consume a significant fraction of the computational time.
When comparing different values of discontinuity bounds (e.g., \emph{u1-0.3} versus \emph{u1-0.2}), we observe that the search is quicker with a larger discontinuity (and a consistently small number of primitives) because it results in a smaller branching factor and fewer states to expand.
The computational time of the optimization component remains roughly constant, even when smaller discontinuity bounds provide a better initial guess.

Comparing across systems (e.g., \emph{u1-0.3} versus \emph{r2-0.5}), the time spent in the search increases with the dimensionality of the system's state space, in line with the theoretical exponential complexity.
The optimization time also increases, but it remains a minor fraction of the total time.

This analysis offers valuable insights for selecting the initial discontinuity bound and the number of primitives in \ALGidbas.
Ideally, the initial number of primitives and the discontinuity bounds should be chosen to be small (primitives) and large (discontinuity) so that 1) the search finds a discontinuous solution quickly, and 2) the optimizer finds a feasible solution.
Thus, our recommendation is to ``choose the largest discontinuity bound that the trajectory optimization can handle effectively''.

\subsubsection{Analysis of Trajectory Optimization with Free Terminal Time}

We evaluate the four trajectory optimization approaches discussed in Sec.
\ref{sec:optimization}: \emph{(i)} joint trajectory and time optimization (\textit{Free-dt}), \emph{(ii)}
hierarchical time search (\textit{Search-T}), \emph{(iii)} model predictive control (\textit{MPC}), and \emph{(iv)} model predictive contouring control (\textit{MPCC}).

We analyze the computational speed, the success, and the cost value in a set of initial guesses of different discontinuities for
some representative problems with different dynamical systems.
A summary of the results is shown in \cref{fig:analysis-timeopt} (more extensive results are available on the project webpage).

\begin{figure}[t]
	\centering
	\includegraphics[width=.40\textwidth]{
		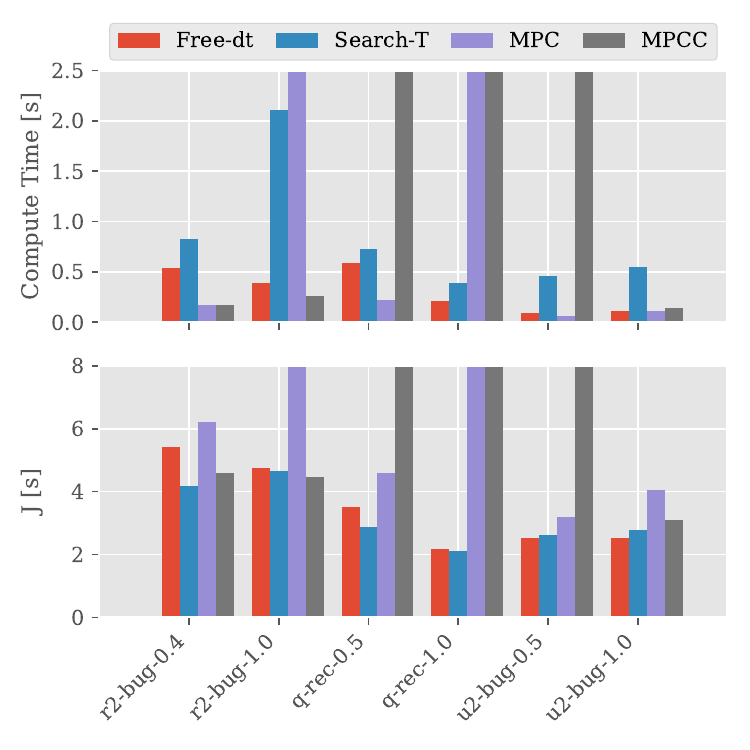}
	\caption{Analysis of four different strategies for trajectory optimization with free terminal time with six different initial guesses (two per problem).
		The label \textit{r2-bug} is short for \textit{Planar rotor -- Bugtrap}, the label \textit{q-rec} is short for \textit{Quadrotor v0 -- Recovery}, and the label \textit{u2-bug} is short for \textit{Unicycle 2 -- Bugtrap}.
		The number after the hyphen indicates the discontinuity bound of each initial guess.
		A bar reaching the top of the plot indicates the algorithm's failure to find a solution.
	}
	\label{fig:analysis-timeopt}
	%
\end{figure}

First, we note that the results highlight a strong variation across the dynamical systems, scenarios, and initial guesses.
We can draw the following general conclusions:

\begin{enumerate}
	\item \emph{Robustness}: \textit{Free-dt} and \textit{Search-T} are more robust and are able to successfully optimize more initial guesses than \textit{MPC} and \textit{MPCC}.
	      On the contrary, \textit{MPC} and \textit{MPCC} are harder to tune and sometimes fail, especially for larger values of the discontinuity bound.
	      The sliding window approaches repair the trajectory locally, step by step, and often cannot reach the final goal if doing so requires jointly improving the initial guess trajectory (where we need to propagate information about the goal across the entire trajectory).
	\item \emph{Computation speed}: When \textit{MPC} and \textit{MPCC} manage to find a solution, they are the fastest methods.
	      Comparing \textit{Free-dt} and \textit{Search-T}, \textit{Free-dt} is between 1.5 and 5 times faster than \textit{Search-T}.
	\item \emph{Cost convergence}: The joint approaches \textit{Free-dt} and \textit{Search-T} converge to a better cost than the sliding window approaches because they consider the full trajectory at once, but the difference is small.
\end{enumerate}

It is worth noting that compute times for the same problem with different discontinuity bounds are not directly comparable since the initial guesses contain a varying number of steps.
Furthermore, these results heavily depend on the underlying optimization algorithm (differential dynamic programming), which effectively addresses the temporal dimension of the trajectory optimization problem (with linear complexity on the number of time steps).

In \ALGidbas, the search component consumes more computational time than the optimization process, as illustrated in \cref{fig:time-spent-each-component}.
In trajectory optimization, we prioritize achieving reliable results and converging to low-cost solutions over computational efficiency.
Consequently, \textit{Free-dt} has been selected as the default optimization algorithm for \ALGidbas, striking an optimal balance between convergence, robustness, and computational speed.

\subsubsection{Analysis of Heuristic Functions}

\begin{figure}[t]
	\centering
	\includegraphics[width=.40\textwidth]{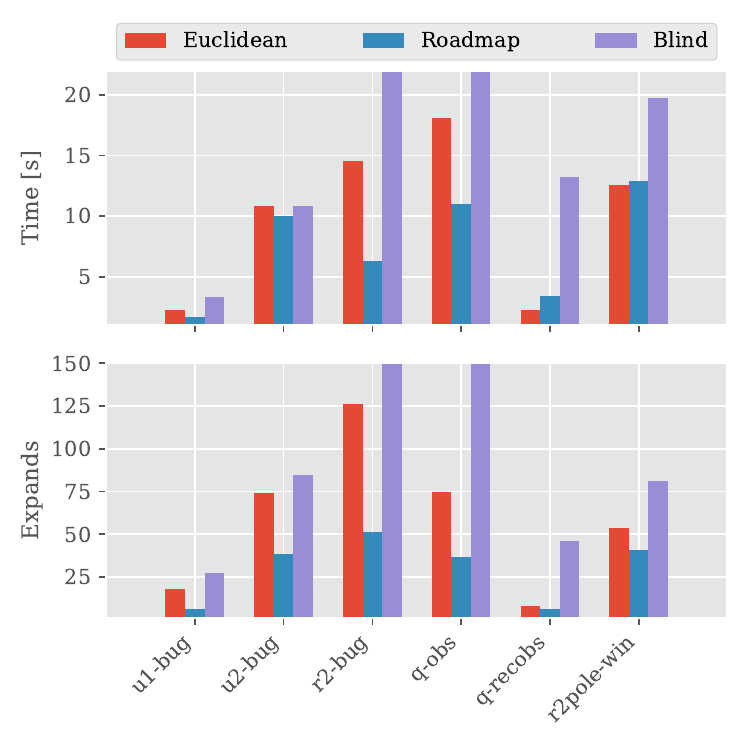}
	\caption{Analysis of the heuristic function in Db-A* across six different problems:
		\textit{Unicycle 1 v0 -- Bugtrap} (\textit{u1-bug}), \textit{Unicycle 2 -- Bugtrap} (\textit{u2-bug}), \textit{Planar rotor -- Bugtrap} (\textit{r2-bug}), \textit{Quadrotor v0 -- Obstacle} (\textit{q-obs}), \textit{Quadrotor v0 -- Recovery obstacles} (\textit{q-recobs}), and \textit{Rotor pole -- Window} (\textit{r2pole-win}).
		The number of expansions is measured in thousands of units; for example, a value of 100 on the graph represents 100,000 expansions.
		A bar reaching the top of the plot indicates the algorithm's failure to find a solution.
	}

	\label{fig:analysis-heu}
\end{figure}

We analyze three different heuristic functions to inform the search in \ALGdb: the Euclidean heuristic (\textit{Euclidean}), the roadmap heuristic (\textit{Roadmap}), and the blind heuristic (\textit{Blind}) (see \cref{sec:Discontinuity-Bounded-Search}).

We report the number of expanded nodes and the computational time for a single search of \ALGdb on six different problems involving obstacles in \cref{fig:analysis-heu} (with more extensive results on our webpage).
We first note that the comparison is highly dependent on the number of motion primitives and the discontinuity bound.
For a large discontinuity bound and a small number of primitives, the search's branching factor is small, and the heuristic is relatively irrelevant; even uninformed exploration can find the goal with few expansions.
With an increased number of primitives, the potential states to expand grow, necessitating a good heuristic function.

To better illustrate the differences between heuristics, we evaluate \ALGdb with smaller discontinuity values (and correspondingly higher numbers of primitives) than those used in the first iteration of \ALGidbas in the main benchmark, resulting in many more expansions and computational time.

The \textit{Roadmap} heuristic does not include precomputation time, which requires building a coarse roadmap and computing the optimal cost-to-go for all nodes.
This process takes approximately \SI{50}{ms} for the unicycles, \SI{70}{ms} for the \textit{Planar rotor} (with roadmaps of 1,000 configurations), and \SI{400}{ms} for the \textit{Quadrotor v0} (with a roadmap of 3,000 configurations).

We observe that the \textit{Roadmap} heuristic is the most informative (as it considers obstacles) and reduces the number of expanded nodes in all selected problems.
However, in some cases, the \textit{Euclidean} heuristic is competitive in terms of compute time (and even faster for some problems), as it is quicker to evaluate because the \textit{Roadmap} requires a k-d tree search with every evaluation.
The \textit{Blind} heuristic fails to solve some problems within the time limit of \SI{30}{s}.
Based on these results, in our algorithm, we select the \textit{Euclidean} heuristic because it requires no precomputation or additional hyperparameters, and it is informative and rapid for all dynamical systems.

\subsubsection{Motion Primitives - Optimization vs.
	Sampling}

We also evaluated the performance of \ALGidbas using motion primitives with random controls (i.e., short trajectories generated by sampling controls uniformly at random within bounds), as opposed to our optimization-based motion primitives.

For car-like robots, the success, convergence, and computation time of \ALGidbas are similar with both strategies.
The cost of the discontinuity-bounded trajectories produced in the search step using primitives with random controls is often higher, but the optimization step successfully improves these trajectories, achieving a similar cost in the final feasible trajectory.
Conversely, motion primitives with random controls often fail to solve problems with flying robots, where such primitives frequently result in unstable final intermediate configurations and erratic trajectories (as exemplified by the performance of \ALGsst with flying robots in the benchmark).

\subsubsection{Motion Primitives -- Invariance and Equivariance}

All systems, except for the \textit{Acrobot}, exhibit a form of invariance or equivariance.
Using invariance, the same primitive can be transformed on-the-fly (e.g., translated) to be applicable in different states with smaller discontinuity values (\cref{sec:inv-equi-search}).
For car-like robots, we leverage only translation invariance (\cref{ex:unicycle:invariance}).
For flying robots, we utilize translation and linear velocity invariance in the second-order dynamics (\cref{ex:quad:invariance}).
Since the number of motion primitives required to attain a given coverage resolution grows exponentially with each dimension, exploiting invariance is crucial for planning with a reduced number of primitives.
For instance, we use only 2,000 primitives for the \qsystem{Quadrotor v0} in the first iteration of \ALGidbas by utilizing translation and velocity invariance.
With only translation invariance, over 50,000 primitives would be required to achieve comparable levels of discontinuity.

\subsection{Limitations and Future Work}

From a practical standpoint, the main limitation of \ALGidbas is that it necessitates a preliminary offline step to generate motion primitives.
Moreover, adding a new dynamical system requires solid theoretical and practical knowledge of two different paradigms in motion planning: search and trajectory optimization, and the interplay between them to choose some important hyperparameters.
In this sense, the method is more complex than sample-based motion planners.

The number of required motion primitives grows exponentially with the state dimension.
To mitigate this issue, a possible solution is to use a more informative distance metric (instead of the weighted Euclidean metric) when deciding which primitive to apply, which correlates better with the underlying dynamics and the subsequent trajectory optimization.
Additionally, more informed sampling strategies for start and goal configurations when generating motion primitives could reduce the number of primitives needed.

To improve the computation time required to find the first solution in some problems (e.g., \SI{1.5}{s} in \textit{Quadrotor v0 - Recovery obstacles} or \SI{12.3}{s} in \textit{Planar rotor - Bugtrap}) and scale to larger environments with more obstacles, we see great potential in combining our discontinuity-based approach with an RRT-like planner, instead of an incremental A* search.
We are also interested in exploring hybrid approaches between our method and the control propagation used in \cite{shomeAsymptoticallyOptimalKinodynamic2021b}.

\section{Conclusion}

We present iDb-A*, a new kinodynamic motion planner that combines a novel graph-search method with trajectory optimization iteratively.
For the graph search, we introduce Db-A*, a generalization of A* that reuses motion primitives to compute trajectories with bounded discontinuity, which are later used as a warm start for trajectory optimization.

iDb-A* amalgamates the ideas and advantages of sampling-based, search-based, and optimization-based kinodynamic motion planners: it converges asymptotically to the optimal solution, rapidly finds a near-optimal solution, and does not require any additional post-processing.

We evaluate iDb-A* on a diverse set of challenging, time-optimal kinodynamic motion planning problems, from obstacle avoidance with car-like robots to highly dynamic maneuvers with quadcopters.
iDb-A* consistently outperforms other algorithms on these benchmarks and solves problems that were beyond the capabilities of previous motion planners.

The main limitation is that the number of motion primitives required grows exponentially with the state dimension, which poses a challenge to systems with higher dimensionality.

Finally, we believe that our combination of search, sampling, and optimization lays the foundation for novel kinodynamic planners for robotic manipulation and contact planning.

\bibliographystyle{IEEEtran}
\bibliography{IEEEabrv,references}

\end{document}